\documentclass[sigconf,screen]{acmart}
\usepackage{cleveref}

\newtheorem{proposition}{Proposition}

\newtheorem{theorem}{Theorem}

\usepackage{algorithm}
\usepackage{algorithmic}
\usepackage{subcaption}
\DeclareMathOperator*{\argmin}{arg\,min}
\DeclareMathOperator*{\argmax}{arg\,max}
\usepackage{bm}
\usepackage{bbm}

\usepackage{comment}
\usepackage{multirow}
\AtBeginDocument{%
  \providecommand\BibTeX{{%
    \normalfont B\kern-0.5em{\scshape i\kern-0.25em b}\kern-0.8em\TeX}}}

\setcopyright{acmcopyright}
\copyrightyear{2021}
\acmYear{2021}
\acmDOI{}

\begin{document}

%\title{Fair Pairwise Learning by Correcting Biases in Datasets}
\title{A Pre-processing Method for Fairness in Ranking}
%\authoranon{
\author{Ryosuke Sonoda}
\email{sonoda.ryosuke@fujitsu.com}
\orcid{1234-5678-9012}
%\email{webmaster@marysville-ohio.com}
\affiliation{%
  \institution{Fujitsu Ltd.}
  \streetaddress{P.O. Box 1212}
  \city{Kanagawa}
  \state{Kawasaki}
  \postcode{123-456}
}
%}


\begin{abstract}
  Fair ranking problems arise in many decision-making processes that often necessitate a trade-off between utility and fairness. 
  Existing methods have tried improving this trade-off using model-dependent algorithms; however, there is no suitable method for arbitrary ranking models.
  In this paper, we propose a data pre-processing method to improve the trade-off of arbitrary ranking models with respect to fairness measurements in ranking.
  Our method is based on a pairwise ordering method, which is promising for ranking because it considers relative orders of data.
  We show that our method for weighting pairs of training data can lead to ranking models equivalent to those trained on pairs of unbiased data.
  We also provide a practical algorithm that yields a stable solution by repeatedly weighting of data and training of ranking models.
  Experiments on benchmark data indicate that the proposed method has better trade-off than existing methods in various fairness measurements in ranking.
\end{abstract}




%  \incluphics[width=\textwidth]{figure/ponchi}

\maketitle

%!TEX root = paper.tex
\section{Introduction}
\label{sec:intro}
\begin{comment}
\begin{figure}[t]
  \begin{center}
    \includegraphics[width=\linewidth]{figure/ponchi}
    \caption{punch}
    \label{fig:punch}
  \end{center}
\end{figure}
\end{comment}

Fairness in ranking is a fundamental problem in information retrieval because a data-driven ranking model is often used in various online services to search for items such as web pages, products, and people~\cite{DBLP:conf/icalp/CelisSV18}.
%The ranking problem considers queries as input data, where each query has a set of items, and each item has a label (i.e., a value to rank with regards to a query).
%The goal is to order items with good labels in a high position to maximize the utility for their queries~\cite{robertson1977probability}.
A ranking problem can be defined as a machine learning (ML) task~\cite{robertson1977probability}. 
Given a collection of pairs of queries and items each annotated by a numeric value that is typically a preference of the item with respect to the corresponding queries, a ranking model is trained so that it outputs an ordered list of items to maximize the utility for a searcher when it is supplied with an unseen query.
Many ordering methods have been proposed to maximize utility~\cite{DBLP:conf/nips/LiBW07,DBLP:conf/icml/ChuK05,burges2010ranknet,DBLP:journals/ftir/Liu09,DBLP:conf/uai/RendleFGS09,DBLP:conf/sigir/CaoXLLHH06}.
In particular, pairwise ordering methods are widely used and promising because they consider the relative order between pairs of items~\cite{DBLP:journals/ftir/Liu09, DBLP:conf/uai/RendleFGS09}.
As the impact of ranking on people's lives is increasing with the spread of ordering methods, there is a need to maximize not only utility but also fairness.
Fairness measurements that quantify bias in a model or data have been formulated by evaluating the difference between collections of items from different groups (e.g., gender, race, or origin) for a query, being aware of their orders~\cite{DBLP:conf/kdd/BeutelCDQWWHZHC19,DBLP:conf/kdd/SinghJ18}.
Because both utility and fairness matter, most fair ranking algorithms studied so far assume access to their ranking model for an optimal trade-off; hence, they are model-dependent~\cite{DBLP:journals/corr/abs-2103-14000,DBLP:journals/corr/abs-2104-05994}.

Many studies have been reported to improve the trade-off between utility and fairness~\cite{DBLP:conf/aaai/NarasimhanCGW20,DBLP:conf/kdd/BeutelCDQWWHZHC19,DBLP:conf/kdd/SinghJ18,DBLP:conf/www/KuhlmanVR19,DBLP:conf/innovations/DworkHPRZ12,DBLP:conf/nips/HardtPNS16, DBLP:conf/aistats/JiangN20}.%, less studies have been made on fair ranking with our best knowledge.
Among them, Jian and Nachum~\cite{DBLP:conf/aistats/JiangN20} proposed a data pre-processing method in fair classification, which assumes the true data that are inaccessible in reality and are free from bias.
Given an arbitrary classification algorithm, this method weights each datum so that the loss becomes equivalent to that with bias-free data, thereby outperforming comparison methods in the literature~\cite{DBLP:conf/innovations/DworkHPRZ12,DBLP:conf/nips/HardtPNS16}. 
Narashimhan et al.~\cite{DBLP:conf/aaai/NarasimhanCGW20} formulated a fair ranking as a constrained optimization of loss function, where both loss function and constraints representing fairness measurements consider the ordering of items proprietary, yielding further promising results than the reported fair ranking algorithms so far~\cite{DBLP:conf/kdd/BeutelCDQWWHZHC19,DBLP:conf/kdd/SinghJ18,DBLP:conf/www/KuhlmanVR19}.

In this study, we propose a pre-processing method of data to turn arbitrary ranking models fair with respect to arbitrary fairness measurements.% under minor constraint.
Our work is nontrivial extension of Jian and Nachum's~\cite{DBLP:conf/aistats/JiangN20} from fair classification to fair ranking.
The contribution of this study is three-fold.
First, we introduce pairwise ordering method, which we believe is prominent for fair ranking problems to ensure parity across groups, rather than pointwise ordering method in the original pre-processing method.
To the best of our knowledge, this is the first proposal of a pre-processing method proprietary for fair ranking, hence agnostic to both ranking model and fairness measurement.
Second, we prove that there exists a closed-form solution to the fair ranking problem over biased data transferred from a ranking problem over unbiased data by weighting the data, provided the fairness measurement is expressed as a linear constraint of a model.
Finally, the existence of closed-form solution encourages us to implement a practical algorithm that iterates the weighting of data and training of a model, thereby yielding a stable solution.
The experimental results demonstrated that the proposed algorithm outperformed existing ones in terms of the trade-off between the  utility and fairness of different measures on real-world datasets.

The remainder of this paper is organized as follows.
We describe works related to our study in Section \ref{sec:related}.
We describe the problem formulation, definition of measurements in ranking, and introduce the re-weighting method in Section \ref{sec:problem}.
We present our proposed method that optimizes the trade-off between utility and fairness in ranking in Section \ref{sec:method}.
We show our experimental settings and results in Section \ref{sec:experiment}.
We conclude in Section \ref{sec:conclude}.
%!TEX root = paper.tex

\section{Related Work}\label{sec:related}
In this section, we introduce related work on ranking algorithms, fairness in ML, and fairness in ranking algorithms.
\subsection{Ranking algorithms}\label{subsec:ltr}
Ranking algorithms are widely recognized for their potential for societal impact \cite{DBLP:journals/cacm/Baeza-Yates18}, as they form the core of many online systems, including search engines, recommendation systems, news feeds, and online voting.
Ranking algorithms typically use ML algorithms to construct ranking models, given a set of query-item pairs.
The goal of a ranking model is to predict the score of new items and then sort the score for each query.
The main challenge of this task is that data for training a ranking model are usually long-tailed, yielding the rich-get-richer problem~\cite{DBLP:journals/ftir/Liu09}.
Thus, many studies have proposed ranking algorithms to solve this problem~\cite{DBLP:conf/nips/LiBW07,DBLP:conf/icml/ChuK05,burges2010ranknet,DBLP:conf/uai/RendleFGS09,DBLP:conf/sigir/CaoXLLHH06}.

The most straightforward algorithm to solve the ranking problem is an algorithm based on pointwise ordering method that aims to predict a label of each item (i.e., a preference of the item with respect to the corresponding queries).
For example, classification- and regression-based algorithms predict the score of an item based on its ground-truth label~\cite{DBLP:conf/nips/LiBW07,DBLP:conf/icml/ChuK05}.
Therefore, the loss function of a pointwise ordering method only considers the relevance between queries and items, not the ranking of the items.
Moreover, pairwise ordering methods that usually outperform pointwise ordering methods have been proposed~\cite{DBLP:conf/sigir/CaoXLLHH06,DBLP:conf/uai/RendleFGS09,burges2010ranknet}.
This is because the key issue of ranking is to determine the orders of items and not to judge the label of items, which is exactly the goal of pairwise ordering method.
In addition, in the long-tailed data, a  pointwise ordering method falls into suboptimal because its loss function is dominated by queries with lots of items~\cite{DBLP:journals/ftir/Liu09}.
Meanwhile, pairwise ordering methods, which consider the order of item pairs in the same query in its loss, can avoid this problem.

Given the above reasons, pairwise ordering methods are more widely applied for ranking problems than pointwise ones.
In this paper, we employ pairwise ordering method.
\subsection{Fairness in ML}\label{subsec:fairness_ml}
Fairness has become a focus of concern in ML algorithms integrated into the decision-making process, such as loan screening, university admissions, and online search.
In ML, fairness is the absence of any prejudice toward certain groups (e.g., gender, race, or origin)~\cite{DBLP:conf/innovations/DworkHPRZ12,DBLP:conf/nips/HardtPNS16,DBLP:conf/kdd/SinghJ18}.
The goal in fair ML is to optimize a trade-off between utility and fairness w.r.t. certain groups.
The main challenge is that data usually contain a bias for certain groups~\cite{implicit-bias}, and an ML model trained on such data can yield unfair predictions for those groups.
To address this challenge, many methods have been proposed to improve the trade-off using fairness measurements as constraints, and these methods can be categorized as post-, in-, and pre-processing methods, based on the processes they perform.
%These measurements are not easy to improve for two reasons.
%First, observed data used to train a machine learning model often contain bias for specific groups.
%A simple way to tackle this bias, i.e., removing the sensitive attribute from data, will not improve fairness measurements as bias is embedded in the data in a complex way~\cite{DBLP:conf/kdd/PedreschiRT08}.
%Second, in general, there is a trade-off between model utility and fairness~\cite{DBLP:conf/kdd/Neto20}.
%Therefore, fairness in machine learning has become a hot issue.

%Many fairness measurements have been proposed to statistically evaluate fairness, and their use varies depending on the application of the ranking model.
%The measurement statistically evaluates the disparity between individuals or groups based on their sensitive attributes.
%While quantifying is simple, the measurement is not uniquely determined and depends on the task \cite{DBLP:conf/fat/Raz21}.
%Therefore, many measurements have been proposed today (e.g., for classification tasks~\cite{DBLP:conf/innovations/DworkHPRZ12,DBLP:conf/nips/HardtPNS16}, for ranking tasks~\cite{DBLP:conf/kdd/SinghJ18,DBLP:conf/kdd/BeutelCDQWWHZHC19}, etc.).
%These measures are typically used not only to evaluate the fairness, but also as constraints to correct biases.
%However, improving fairness through constraints often creates the trade-off with utility, and thus needs to be analyzed \cite{DBLP:conf/kdd/Neto20}.

Post-processing methods recognize that predictions from an ML model may be unfair to certain groups.
The goal of the post-processing methods is to satisfy the constraints by appropriately changing the predictions from an ML model~\cite{DBLP:conf/kdd/SinghJ18,DBLP:conf/nips/HardtPNS16}.
These methods address the problem of fairness without changing an ML model or training data by assuming that an ML model is learned accurately and that the trade-off is achieved only by changing its predictions.
%This approach still has the problem of requiring access to the sensitive attributes of items at test time.
%These methods are easy to optimize because they use only the predictions from a model and protected attributes of test dataset instead of all features of the dataset.引用！
%However, excluding all features from the optimization, may not only degrade the utility \cite{DBLP:conf/colt/WoodworthGOS17}, but also fail to optimize the fairness because biases may be embedded in all features \cite{DBLP:conf/nips/SinghJ19}.

Many works have proposed in-processing methods that tackle fairness in a training process by adding constraints to model loss~\cite{DBLP:journals/jmlr/CotterJGWNYS19,DBLP:conf/icml/CotterGJSSWWY19,DBLP:conf/aistats/ZafarVGG17}.
%These methods have two advantages over the post-processing methods.
%First, the in-processing methods do not require sensitive attributes of the items at prediction time, which is a practical setting \cite{DBLP:conf/aies/BeutelCDQWLKBC19}.
The advantage of in-processing methods is that they can apply the method of Lagrange multipliers to transform constraints to penalties.
Thus, the in-processing methods can optimize an ML model parameters and Lagrangian multipliers simultaneously.
For this reason, in general, the in-processing methods yield a better trade-off than post-processing methods~\cite{DBLP:journals/jmlr/CotterJGWNYS19,DBLP:conf/icml/CotterGJSSWWY19}.
%However, excluding all features from the optimization, may not only degrade the utility \cite{DBLP:conf/colt/WoodworthGOS17}, but also fail to optimize the fairness because biases may be embedded in all features \cite{DBLP:conf/nips/SinghJ19}.
%The in-processing methods approximate the constraints, typically non-differentiable, to make the joint optimization feasible.
%However, this approximation increases the complexity of the ranking model, making the optimization challenging and unstable to converge \cite{DBLP:journals/jmlr/CotterJGWNYS19}, and worsening the interpretability of the ranking model \cite{DBLP:conf/icml/CotterGJSSWWY19}.

Pre-processing methods address the problem of bias in training data, which is open for post- and in-processing methods~\cite{DBLP:conf/innovations/DworkHPRZ12}.
The core of pre-processing methods is to create a ``fair'' data by modifying an original training data.
%These methods have an advantage of the process due to model agnostic.
Contrary, to post- and in-processing methods, pre-processing methods have advantages of stability and simplicity because their procedure is model agnostic.
In-processing methods approximate nondifferential constraints, causing high complexity and unstable convergence of ML models~\cite{DBLP:journals/jmlr/CotterJGWNYS19}, but pre-processing methods do not require such approximations and are therefore more practical to fairness.

%These method can be achieved by simply correcting the bias of labels and features in the training data.
%However, this method is open to the problem of automatically estimating how much data bias should be corrected to achieve a tradeoff.
%In contrast to post- and in-, the pre-processing methods directly address the biases in the training dataset, which is the cause of unfairness in a machine learning model \cite{DBLP:conf/innovations/DworkHPRZ12}.
%These methods try to improve the the trade-off of the ranking model by weighting the features or labels of the training dataset.
%The advantage of these methods is model agnostic, and the processing to enforce fairness is simpler than the post- and in-processing methods.
%The disadvantage of these methods is that they do not impose any constraints on the ranking model, making it difficult to guarantee the trade-off.

In this paper, we consider the problem of automatically estimating how much to correct bias in training data, which has not been addressed in existing pre-processing methods.
In other words, we aim to improve the trade-off of pre-processing methods over post- and in-processing methods.
\subsection{Fairness in Ranking Algorithms}\label{subsec:fair_ranking}
Developing algorithms for ranking fairness is a recent area of research in ML.
The goal of this area is to develop ML algorithms to achieve the trade-off in ranking.

Many works have proposed algorithms for fair ranking~\cite{DBLP:conf/kdd/SinghJ18,DBLP:conf/kdd/BeutelCDQWWHZHC19,DBLP:conf/www/KuhlmanVR19,DBLP:conf/aaai/NarasimhanCGW20,DBLP:conf/aistats/JiangN20}.
In particular, Singh and Joachims~\cite{DBLP:conf/kdd/SinghJ18} considered the fairness of rankings through the lens of exposure allocation between groups.
Instead of defining a single measurement of fairness, they developed a general framework that employs post-processing and linear programming to optimize an utility-maximizing ranking under a class of constraints.
%Singh and Joachims \cite{DBLP:conf/kdd/SinghJ18} apply the post-processing method on predicted labels: First, a linear regression model is trained on all the training dataset query-item pairs that predict their labels.
%For each query in the test dataset, they use the predicted labels of the items as an input to the linear program with the listwise fairness.
%This method has been demonstrated to be a better trade-off than existing heuristic approaches (e.g., \cite{DBLP:conf/cikm/ZehlikeB0HMB17}).
Beutel et al.~\cite{DBLP:conf/kdd/BeutelCDQWWHZHC19} also provided multiple definitions of fairness measurements in rankings through pairwise data, called pairwise fairness.
They presented an in-processing method using a fixed constraint that correlated to pairwise fairness.
%Based on pairwise data, they can evaluate fairness in rankings to see if a model systematically mis-ranks or under-ranks items from a particular group.
%We show that this measure aligns with ranking fairness definitions but is not covered by pointwise fairness measures.
%We ultimately offer a novel pairwise regularization approach to improve recommender system fairness during training.
Rather than using fixed constraints, Narasimhan et al.~\cite{DBLP:conf/aaai/NarasimhanCGW20} proposed an in-processing method that can handle a class of pairwise fairness as a constraint.
%Narasimhan et al. \cite{DBLP:conf/aaai/NarasimhanCGW20} showed that pairwise fairness can be intuitively defined to handle supervised and unsupervised measurements of fairness, for ranking and regression, and for discrete and continuous protected attributes.
Further, they showed that their in-processing method outperforms methods in the literature~\cite{DBLP:conf/kdd/SinghJ18,DBLP:conf/kdd/BeutelCDQWWHZHC19} in various experimental settings.
%They also showed how pairwise fairness metrics could be incorporated into training using state-of-the-art in-processing methods \cite{DBLP:conf/alt/CotterJS19,DBLP:journals/jmlr/CotterJGWNYS19}.
Recently, Jiang and Nachum~\cite{DBLP:conf/aistats/JiangN20} presented a pre-processing method based on fair classification algorithm, assuming that there exists an unbiased ground truth.
Their pre-processing method re-weights training data points to make them unbiased toward the ground truth.
The contribution of this work is that their method could estimate the appropriate weights for each data point.
%Since the re-weighting method has been proposed for classification problems, its performance on fair ranking problems has not yet been reported.

No Existing study has applied pre-processing method considering orders of items to fair ranking problems; thus, the trade-off is still a problem.
%Existing studies have not addressed data bias and ranking problems simultaneously; thus, the trade-off is still a problem.
Therefore, in this paper, we propose a pre-processing based on pairwise fairness by making the above re-weighting method pairwise.
%!TEX root = paper.tex
\section{Problem Formulation}
\label{sec:problem}

%Consider a standard ranking setup:
In this section, we introduce the settings for pointwise and pairwise ordering methods.
The setting for the pointwise ordering method is an extension of the setting introduced by~\cite{DBLP:conf/aistats/JiangN20} for ranking, whereas we introduce a new setting for the pairwise ordering method.
\subsection{Settings for Pointwise Ordering Method}
Given queries $q \in Q$ drawn from an underlying distribution $\mathcal{D}$, where each query $q$ has a set of items $\mathcal{R}_q$ to be ranked.
Consider a feature space $\mathcal{X}\subseteq \mathbb{R}^d$ where $d$ is a dimension of the feature space and an associated feature distribution $\mathcal{P}$.
Then each item $i\in \mathcal{R}_q$ is represented by an associated vector $x_i\in \mathcal{X}$ and a label $y_i\in \mathcal{Y}$ (e.g., for $\mathcal{Y}:=\{0,1\}$: $y_i = 1$ if query $q$ and item $i$ are relevant, $y_i = 0$ otherwise).
In this paper, we use a binary label setting.
However, this setting may be readily generalized to other settings (e.g., $\mathcal{Y} \in \mathbb{R}$).

We assume the existence of an unbiased, ground-truth label function $y_\mathrm{true} \colon Q\times \mathcal{X} \rightarrow \left[0,1\right]$.
We usually do not have access to the ``true''  values of the function $y_\mathrm{true}$.
Thus, $y_\mathrm{true}$ is the assumed ground truth.
The labels of our dataset are generated based on a biased label function $y_\mathrm{bias}\colon Q\times \mathcal{X} \rightarrow \left[0,1\right]$.
%Accordingly, we assume that we obseve the label $y_i$ from $y_i \sim \mathrm{Bernoulli}(y_\mathrm{bias}(q, x_i))$.
Accordingly, we assume that our data for pointwise setting is drawn as follows:
\begin{equation}
  (q, x, y) \sim \mathcal{D} \equiv q \sim Q, x \sim \mathcal{P}, y \sim \mathrm{Bernoulli}(y_\mathrm{bias}(q, x)).
\end{equation}

In the following, we introduce utility and fairness measurements that the pointwise ordering method can handle.
%These measurements are especially for the fairness classification problem.
The conventional goal in pointwise ordering method is to find a ranking model $h\colon Q\times\mathcal{X}\rightarrow \mathbb{R}$ that maximizes the expected utility:
\begin{equation}
  \label{eq:error_rate}
    \argmax_{h} \mathbb{E}_{q\sim Q}\left[P(h_q(x_i) = y_\mathrm{true}(q,x_i)) \right],
\end{equation}
where $h_q(x_i) = h(q, x_i)$.
In addition, we define the expected bias of $h$ based on a class of constraints for pointwise constraints $c^\mathrm{point}$ introduced~\cite{DBLP:journals/jmlr/CotterJGWNYS19,DBLP:conf/aistats/JiangN20}:
\begin{equation}
  \label{eq:pointwise_fairness}
    \mathbb{E}_{q\sim Q}[\mathbb{E}_{x_i\sim \mathcal{P}}[\langle h_q( x_i), c^\mathrm{point}(q, x_i)\rangle]],
\end{equation}
where $\langle h_q( x_i) c^\mathrm{point}(q, x_i)\rangle:=\sum_{y_i \in \mathcal{Y} }h_q(y_i\mid x_i)c^\mathrm{point}(q, x_i, y_i)$.
We use the shorthand $h(y_i\mid x_i)$ to denote the probability of sampling $y_i$ from a Bernoulli random variable with $p = h_q( x_i)$;
i.e., $h_q(y\mid x_i) := h_q( x_i)$ and $h(0\mid x_i) :=1- h_q( x_i)$.
For $c^\mathrm{point}$, one can employ statistical parity fairness~\cite{DBLP:conf/innovations/DworkHPRZ12}, equal opportunity fairness~\cite{DBLP:conf/nips/HardtPNS16}, etc.
If the fairness is ideal, Eq.~\eqref{eq:pointwise_fairness} is $0$.
When $h$ is biased, some amount is given by Eq.~\eqref{eq:pointwise_fairness}.

\subsection{Settings for Pairwise Ordering Method}
In this section, we introduce the settings for our pairwise ordering method.
In the pairwise setting, we consider item pairs $(i,j) \in \mathcal{R}_q$ from dataset $\mathcal{D}$.
Then we define the binary pair label $l_{ij}$, which represents the relative order of pair $(y_i,y_j)$ (i.e., $l_{ij} = 1$ if $y_i > y_j$, and $l_{ij} = 0$ if $y_i < y_j$).
Here, we introduce a new concept, a pair label function, to consider the biases of pair labels.
We assume the existence of an unbiased, ground-truth pair label function $l_\mathrm{true} \colon Q\times \mathcal{X}^2 \rightarrow \left[0,1\right]$.
%$l_\mathrm{true} := \mathbbm{1}[y_\mathrm{true}(x_i) > y_\mathrm{true}(x_j)]$.
We do not have the actual values for the ground-truth pair label function $l_\mathrm{true}$.
Instead, we observe the biased pair label function $l_\mathrm{bias} \colon Q\times \mathcal{X}^2 \rightarrow \left[0,1\right]$.
Accordingly, we assume that our data for pairwise settings $\mathcal{D}_\mathrm{pair}$ is collected as follows:
\begin{align}
  \begin{split}
  &(q, x_i, x_j, l_{ij}) \sim \mathcal{D}_\mathrm{pair} \\
  &\equiv q \sim Q, (x_i,x_j) \sim \mathcal{P}^2, l_{ij} \sim \mathrm{Bernoulli}(l_\mathrm{bias}(q, x_i,x_j)).
  \end{split}
\end{align}
In the following, we use $x_{ij}$ to denote $(x_i,x_j)$ for simplicity.

Now, we introduce utility and fairness measurements that the pairwise ordering method can handle.
For utility, we employ the area under the curve (AUC) in this paper.
The conventional goal in pairwise ordering method is to find a ranking model $h\colon Q\times\mathcal{X}\rightarrow \mathbb{R}$ that maximizes the expected AUC:
\begin{equation}
  \label{eq:auc}
    \argmax_{h} \mathbb{E}_{q\sim Q}\left[P(h_q(x_i) > h_q(x_j) \mid y_\mathrm{true}(q,x_i) > y_\mathrm{true}(q,x_j)) \right],
\end{equation}
We will hide the term $\mathbb{E}_{q\sim Q}$ and only consider comparisons among relevant items for all following definitions.

Instead of single-mindedly maximizing this utility measurement, we include fairness measurements into the evaluation of the ranking model $h$.
%In fair ranking, many constraint functions have been proposed \cite{DBLP:conf/innovations/DworkHPRZ12,DBLP:conf/nips/HardtPNS16,DBLP:conf/kdd/SinghJ18,DBLP:conf/kdd/BeutelCDQWWHZHC19,DBLP:conf/sigir/BiegaGW18,DBLP:conf/aaai/NarasimhanCGW20}.
We focus on constraints for pairwise  constraints \cite{DBLP:conf/kdd/BeutelCDQWWHZHC19,DBLP:conf/aaai/NarasimhanCGW20}. %(see Section \ref{subsec:experimental_setup} for details).
However, our method can apply to other constraints (e.g., listwise constraints \cite{DBLP:conf/kdd/SinghJ18}).
%These functions are typically used to evaluate the disparity between groups in rankings by considering the orders of items in each query.

We first introduce a predicted pair label function $\hat{l}$ to present the general class of pairwise constraints.
So, let us denote $\hat{l}_q(x_{ij}) \in \left[0,1\right]$ as a probability of the difference between the ranking model output for $h_q( x_i)$ and $h_q( x_j)$.
For concreteness, $\hat{l}_q(x_{ij}) = \sigma(h_q( x_i)-h_q( x_j))$ where $\sigma(x) = \frac{1}{1+e^{-x}}$ is the sigmoid function.
This definition of $\hat{l}$ allows us to define the class of pairwise constraints $c^\mathrm{pair}$.% as pointwise constraints Eq.~\eqref{eq:pointwise_fairness}.
The pairwise constraints $c^\mathrm{pair}$ may be expressed or approximated as linear constraints on $\hat{l}$.
That is, we define the expected bias of $h$ based on a class of pairwise constraints $c^\mathrm{pair}$:
\begin{equation}
\label{eq:pairwise_bias}
    \Delta = \mathbb{E}_{x_{ij}\sim \mathcal{P}^2}[\langle \hat{l}_q(x_{ij}), c^\mathrm{pair}(q, x_{ij})\rangle],
\end{equation}
where $\langle \hat{l}_q(x_{ij}) c^\mathrm{pair}(q, x_{ij})\rangle:=\sum_{l_{ij} \in \mathcal{Y} }\hat{l}_q(l_{ij}\mid x_{ij})c^\mathrm{pair}(q, x_{ij},l_{ij})$
and we use the shorthand $\hat{l}_q(l_{ij}\mid x_{ij})$ to denote the probability of sampling $l_{ij}$ from a Bernoulli random variable with $p = \hat{l}_q(x_{ij})$;
i.e., $\hat{l}(1\mid x_{ij}) := \hat{l}_q(x_{ij})$ and $\hat{l}(0\mid x_{ij}) :=1- \hat{l}_q(x_{ij})$.
Therefore, a pair label function $\hat{l}$ is unbiased with respect to the constraint function $c^\mathrm{pair}$ if $\Delta=0$. If $\hat{l}$ is biased, the degree of bias (positive or negative) is given by $\Delta$.

We define the constraints $c^\mathrm{pair}$ with respect to a pair of protected groups $(G_k, G_l)$ and thus assume access to an indicator function $g_k(x) = \mathbbm{1}[x \in G_k]$, where $k\in [K]\subseteq |X|$.
We use $Z_{G_{kl}} := \mathbb{E}_{x_{ij}\sim \mathcal{P}^2}[g_k(x_i)\cdot g_l(x_j)]$ to denote the probability of a sample drawn from $\mathcal{P}^2$ to be in $(G_k, G_l)$.
We use $P_{X_{ij}} = \mathbb{E}_{x_{ij}\sim \mathcal{P}^2} [l_\mathrm{true}(q, x_{ij})]$ to denote the proportion of $\mathcal{X}^2$ that is positively labeled and $P_{G_{kl}} = \mathbb{E}_{ x_{ij}\sim \mathcal{P}^2} [g_k(x_i)\cdot g_l(x_j)\cdot l_\mathrm{true}(q, x_{ij})]$ to denote the proportion of $\mathcal{X}^2$ that is positively labeled and in $(G_k, G_l)$.
We now give some concrete examples of accepted notions of constraint functions; however, for all constraint functions, $c^\mathrm{pair}_{kl}(q, x_{ij},0) = 0$.
\begin{definition}\label{def:pair-statistical}
For any $k \neq l$, we define a pairwise statistical parity constraint that requires that if two items are compared from different groups, then, on average, each group has an equal chance of being top-ranked.
%$c^\mathrm{pair}_{kl}(q, x_{ij}, 1)  = \frac{g_k(x_i) g_l(x_j)}{Z_{G_{kl}}} - 1$.
 \begin{equation}
  \label{eq:pair-statistical}
    c^\mathrm{pair}_{kl}(q, x_{ij}, 1)  = \frac{g_k(x_i) g_l(x_j)}{Z_{G_{kl}}} - 1.
 \end{equation}
\end{definition}
\begin{definition}\label{def:pair-inter}
For any $k \neq l$, we define a pairwise inter-group constraint that requires pairs of two items from different groups to be equally likely to be ranked correctly.
%$c^\mathrm{pair}_{kl}(q, x_{ij}, 1)  = l_\mathrm{true}(q, x_{ij}) \left(\frac{g_k(x_i)  g_l(x_j)}{P_{G_{kl}}} - \frac{1}{P_{X_{ij}}} \right)$.
 \begin{equation}
  \label{eq:pair-inter}
    c^\mathrm{pair}_{kl}(q, x_{ij}, 1)  = l_\mathrm{true}(q, x_{ij}) \left(\frac{g_k(x_i)  g_l(x_j)}{P_{G_{kl}}} - \frac{1}{P_{X_{ij}}} \right).
 \end{equation}
\end{definition}
\begin{definition}\label{def:pair-intra}
For any $k = l$, we define a pairwise intra-group constraint that requires pairs of two items from the same group to be equally likely to be ranked correctly.
%$c^\mathrm{pair}_{kl}(q, x_{ij}, 1)  = l_\mathrm{true}(q, x_{ij}) \left(\frac{g_k(x_i)  g_k(x_j)}{P_{G_{kk}}} - \frac{1}{P_{X_{ij}}} \right)$.
 \begin{equation}
  \label{eq:pair-inter}
    c^\mathrm{pair}_{kl}(q, x_{ij}, 1)  = l_\mathrm{true}(q, x_{ij}) \left(\frac{g_k(x_i)  g_k(x_j)}{P_{G_{kk}}} - \frac{1}{P_{X_{ij}}} \right).
 \end{equation}
\end{definition}
\begin{definition}\label{def:pair-marginal}
We define a pairwise marginal constraint that requires pairs to be equally likely to be ranked correctly for each group averaged over the other groups.
%$c^\mathrm{pair}_{kl}(q, x_{ij}, 1)  = l_\mathrm{true}(q, x_{ij}) \left(\frac{g_k(x_i)}{\sum_{l\in [K]}P_{G_{kl}}} - \frac{1}{P_{X_{ij}}} \right)$.
 \begin{equation}
  \label{eq:pair-inter}
    c^\mathrm{pair}_{kl}(q, x_{ij}, 1)  = l_\mathrm{true}(q, x_{ij}) \left(\frac{g_k(x_i)}{\sum_{l\in [K]}P_{G_{kl}}} - \frac{1}{P_{X_{ij}}} \right).
 \end{equation}
\end{definition}

\subsection{Pointwise Modeling Biases in Dataset}\label{subsec:bias}
In this section, we introduce a pre-processing method that is most related to our method. Jiang and Nachum~\cite{DBLP:conf/aistats/JiangN20} assumed that a given biased label function $y_\mathrm{bias}$ is closest to an ideal unbiased label function $y_\mathrm{true}$ in terms of KL-divergence $D_{KL}$.
\begin{definition}
\label{pointwise_loss_function}
A pointwise loss function subjected to a pointwise constraint for some $\epsilon_k \in \mathbb{R}$ is given by
\begin{align}\label{eq:point_loss}
\begin{split}
  &\argmin_{h}  \mathbb{E}_{x_i\sim \mathcal{P}}\left[D_{KL}(h_q( x_i)\|y_\mathrm{true}(q, x_i))\right]\\
  &\text{\ s.t.\ } \mathbb{E}_{x_i\sim \mathcal{P}}\left[ c^\mathrm{point}(q, x_i) \right] = \epsilon_k \text{\ \rm{for all} $k$}.
\end{split}
\end{align}
\end{definition}
%where $c^\mathrm{point}_k(h,G_k\mid q)$ is a pointwise constraint (e.g., a demographic parity constraint \cite{DBLP:conf/innovations/DworkHPRZ12} and an equal of opportunity constraint \cite{DBLP:conf/nips/HardtPNS16}).

According to this KL-divergence loss function with the pointwise constraint, \cite{DBLP:conf/aistats/JiangN20} derived the relationship between $y_\mathrm{bias}$ and $y_\mathrm{true}$ based on the standard theorem \cite{botev2011generalized} (see Appendix in \cite{DBLP:conf/aistats/JiangN20} for proof).
\begin{proposition}
\label{pointwise_proposition}
Based on Definition \ref{pointwise_loss_function}, $y_\mathrm{bias}$ satisfies the following for all $x_i \in \mathcal{X}$.
\begin{align}
    \begin{split}
    \label{eq:pointwise_weight}
    &y_\mathrm{true}(y\mid q,x) \propto y_\mathrm{bias}(y\mid q,x) \cdot \exp{\left(\sum^K_{k=1}\lambda_k c^\mathrm{point}_k(q,x,y) \right)},\\
    &\text{\rm{ for some} $\lambda_1,\dots,\lambda_K \in \mathbb{R}$.}
    \end{split}
\end{align}
\end{proposition}
Based on Proposition~\ref{pointwise_proposition}, \cite{DBLP:conf/aistats/JiangN20} proposed a pre-processing method to recover $y_\mathrm{true}$ by re-weighting $y_\mathrm{bias}$ by the inverse of the second term on the right-hand side of Eq.~\eqref{eq:pointwise_weight}.
In other words, this pre-processing method minimizes the weighted KL-divergence loss function using re-weighted biased labels.
%based on the pointwise ordering method

%This pre-processing method is based on the pointwise ordering method.
%In a ranking problem, the pairwise ordering method that considers the label orders of a given pair of items was found to be more accurate for the task and empirically outperforms the pointwise ordering method \cite{DBLP:conf/uai/RendleFGS09,burges2010ranknet}.
%In addition, the pointwise ordering method can not fully optimize the fairness measurements for ranking, such as pairwise fairness \cite{DBLP:conf/kdd/BeutelCDQWWHZHC19}.
%Thus, a pairwise re-weighting method for unbiased pairwise learning should be developed to improve the fair ranking quality from biased labels.
%!TEX root = paper.tex
\section{Method}
\label{sec:method}
In this section, we explain our pre-processing method based on pairwise ordering method for solving fair ranking problems.
First, in Subsection~\ref{subsec:pairwise_model}, we define our loss function as a KL-divergence loss function with constraints using a pair label function.
According to this loss function, we can derive a closed-form expression for the true pair label function $l_\mathrm{true}$ in terms of the biased pair label function $l_\mathrm{bias}$, the coefficients $\lambda_{11},\dots, \lambda_{KK}$, and the constraint functions $c^\mathrm{pair}_{11},\dots,c^\mathrm{pair}_{KK}$.
In Subsection~\ref{subsec:proposed}, we present how to use this closed-form expression to weight observed data.
Subsequently, we present a weighted loss function by reformulating the constrained KL-divergence loss function using these weights.
In Subsection~\ref{subsec:coefficients}, we present our algorithm that minimizes this weighted loss function.
We show how this minimization optimizes the trade-off between utility and fairness in rankings.
Finally, we describe how to extend our method to more general measurements of fairness in rankings in Subsection~\ref{subsec:extension}.
%We have derived a closed-form expression for the true, unbiased pair label function $l_\mathrm{true}$ in terms of the observed pair label function $l_\mathrm{bias}$, the coefficients $\lambda_{11},\dots, \lambda_{KK}$, and the constraint functions $c_{11},\dots,c_{KK}$.
%In this section, we discuss how to learn a ranking model $h$ to fit $l_\mathrm{true}$, given access to pairs of items $\mathcal{X}\times \mathcal{X}$ with pair labels sampled according to $l_\mathrm{bias}$.
%We begin by restricting ourselves to the constraints $c_{11},\dots,c^\mathrm{pair}_{KK}$ associated with Demo-List fairness, allowing us to have full knowledge of these constraint functions.
\subsection{Pairwise Modeling Biases in Dataset}
\label{subsec:pairwise_model}
We now introduce our mathematical framework to evaluate bias in pairs of items by deriving the relationship between $l_\mathrm{true}$ and $l_\mathrm{bias}$.
One key point of this derivation is to formulate a constrained KL-divergence loss function using the paired dataset $\mathcal{D}_\mathrm{pair}$.
We formulate the following constrained optimization problem.
\begin{definition}
\label{definition}
Now, we assume that there exist $\epsilon_{kl} \in \mathbb{R}$ such that the observed, biased pair label function $l_\mathrm{bias}$ is the solution of the following constrained optimization problem:
\begin{align}
\begin{split}
  \label{eq:our_loss}
  &\argmin_{h}  \mathbb{E}_{x_{ij}\sim \mathcal{P}^2}\left[D_{KL}(\hat{l}_q(x_{ij})\|l_\mathrm{true}(q,x_{ij}))\right] \\
  &\text{\ s.t.\ } %\mathbb{E}_{q\sim Q}\left[ c^\mathrm{pair}_{kl}(q, x_{ij}) \right] = \epsilon_{kl} \text{\ \rm{for all} $k,l$}
  \mathbb{E}_{x_{ij}\sim \mathcal{P}^2}[\langle \hat{l}_q(x_{ij}), c^\mathrm{pair}(q, x_{ij})\rangle] = \epsilon_{kl} \text{\ \rm{for all} $k,l$}.
\end{split}
\end{align}
\end{definition}
This pairwise ordering method differs from the pointwise one Eq. \eqref{eq:point_loss} in terms of sampling items.
%Also, our approach can handle the ranking fairness measurements directly. %as in Eq. \eqref{eq:fair_ranking}.

According to the above KL-divergence loss function with constraints, we can derive a closed-form expression for $l_\mathrm{bias}$ in terms of $l_\mathrm{true}$ using the same procedure  in previous studies~\cite{botev2011generalized,DBLP:journals/tit/FriedlanderG06,DBLP:conf/aistats/JiangN20}.
We derive the following relationship.%(We See Appendix for completeness. point label と pair label を書き換えただけで証明は変わらない).
\begin{proposition}
\label{proposition}
Based on Definition \ref{definition}, $l_\mathrm{bias}$ satisfies the following for all $x_{ij} \in \mathcal{X}^2$.
\begin{align}
  \begin{split}
  \label{eq:pairwise_weight}
  &l_\mathrm{true}(l_{ij}\mid q,x_{ij}) \propto l_\mathrm{bias}(l_{ij}\mid q,x_{ij}) \cdot \exp{\left(\sum^K_{k,l=1}\lambda_{kl}c^\mathrm{pair}_{kl}(q, x_{ij},l_{ij})\right)},\\
  &\text{\rm{ for some} $\lambda_{11},\dots,\lambda_{KK} \in \mathbb{R}$.}
\end{split}
\end{align}
\end{proposition}
Based on Eq.~\eqref{eq:pairwise_weight}, we propose a pre-processing method to recover $l_\mathrm{true}$ by weighting $l_\mathrm{bias}$ by the inverse of the second term on the right-hand side of Eq.~\eqref{eq:pairwise_weight}.
In other words, this weighting method minimizes the weighted KL-divergence loss function using the weighted observed pair labels.

%First, in a ranking problem, our loss function over pairs of labels was found to be more accurate for the task and empirically outperforms the loss function over points of labels~\cite{DBLP:conf/uai/RendleFGS09,burges2010ranknet}.
%This pre-processing method is the pairwise ordering method.
%This approach generally predicts the rankings more accurately than pointwise approaches.
%In addition, this can directly use the fairness measurements for the ranking.
%Therefore, our method can achieve a better trade-off than the pointwise method like \ref{pointwise_loss_function} in the fair ranking problem.
Compared with the existing framework~\cite{DBLP:conf/aistats/JiangN20}, our framework has advantages for optimizing the trade-off between utility and fairness in ranking. 
The existing framework consider weight Eq.~\eqref{eq:pointwise_weight} to evaluate fairness using pointwise constraint $c^\mathrm{point}$.
Meanwhile, our weight Eq.~\eqref{eq:pairwise_weight} use pairwise constraint $c^\mathrm{pair}$ that represents an approximation of fairness measurements in ranking.
Thus, our weights can capture the relativity of the ranking, and we can easily determine what the value of $c^\mathrm{pair}$ means in improving the fairness measurement in ranking.
%However, the pointwise basis constraint ignores the relativity of the rankings, and thus it is hard to determine what an improvement of $c^\mathrm{point}$ means for fairness measurements in ranking~\cite{DBLP:conf/kdd/BeutelCDQWWHZHC19}.
In the next section, we show that we can use this weight to provide a weighted loss function that is unbiased for the ground truth.

\subsection{Proposed Unbiased Loss Function}
\label{subsec:proposed}
In this section, we present an unbiased loss function, i.e., the weighted loss function using the observed pair label function $l_\mathrm{bias}$.
For simplicity, we first present how a ranking model $h$ may be learned assuming knowledge of the coefficients $\lambda_{11},\dots, \lambda_{KK}$.
%We subsequently show how the coefficients themselves may be learned, thus allowing our algorithm to be used in a general setting (Section \ref{subsec:coefficients}).
We now have the closed-form expression Eq. \eqref{eq:pairwise_weight} for the true pair label function $l_\mathrm{true}$.
%In practice, we do not have access to the values $l_\mathrm{true}$.
%We only access pair labels sampled from $l_\mathrm{bias}$.
Based on this expression, we propose a weighting technique to train $h$ on pair labels based on $l_\mathrm{true}$.
This weighting technique weights a pair of items $x_{ij}$ using the weight $w(x_{ij}, l_{ij})$:
\begin{equation}
\label{eq:weight}
w(x_{ij}, l_{ij}) = \frac{\tilde{w}(x_{ij}, l_{ij})}{\sum_{l^{\prime}_{ij}\in \mathcal{Y}} \tilde{w}(x_{ij},l^{\prime}_{ij})},
\end{equation}
where % \eqref{eq:weight}
%\begin{equation} \label{eq:weight}
%  w(x_{ij},l_{ij}) = \begin{cases}
%  %\label{eq:weight}
%    \sigma(\tilde{w}(x_{ij},l^'_{ij})) & \text{if }l_{ij}=1, \\
%    1-\sigma(\tilde{w}(x_{ij},l^'_{ij})) & \text{otherwise}
%  \end{cases}
%\end{equation}
%where
\begin{equation}
  \tilde{w}(x_{ij},l^{\prime}_{ij}) = \exp{\left(\sum_{k,l\in [K]}\lambda_{kl}c^\mathrm{pair}_{kl}(q, x_{ij}, l^{\prime}_{ij})\right)}.%,\\
  %x_i\in G_k, x_j\in G_l.
\end{equation}
%In the following sections, we denote $x_{ij}$ as $(x_i,x_j)$ for simplicity.

We have the following theorem, which states that training a ranking model based on pairs of items with biased pair labels weighted by $w(x_{ij},l_{ij})$ is equivalent to training a ranking model on pairs of items labeled according to the true, unbiased pair labels.
\begin{theorem}
\label{thm:unbiased}
For any loss function $L$ from the paired dataset $\mathcal{D}_\mathrm{pair}$, training a ranking model $h$ on the weighted objective\\
$\mathbb{E}_{x_{ij}\sim \mathcal{P}^2,l_{ij} \sim l_\mathrm{bias}(q,x_{ij})}\left[w(x_{ij},l_{ij})\cdot L(\hat{l}_q(x_{ij}), l_{ij})\right]$ is equivalent to training the ranking model on the objective\\
$C\cdot \mathbb{E}_{x_{ij}\sim \tilde{\mathcal{P}},l_{ij} \sim l_\mathrm{true}(q,x_{ij})}\left[L(\hat{l}_q(x_{ij}), l_{ij})\right]$
with respect to the underlying, true pair labels, for some distribution $\tilde{\mathcal{P}}$ over $\mathcal{X}^2$
\end{theorem}
\begin{proof}
For a given $x_{ij}$ and for any $l_{ij}$, based on Proposition~\ref{proposition}, we have
\begin{equation}
  w(x_{ij},l_{ij})l_\mathrm{bias}(l_{ij}\mid q,x_{ij}) = \phi(x_{ij})l_\mathrm{true}(l_{ij}\mid q,x_{ij}),
\end{equation}
where $\phi(x_{ij})=\sum_{l_{ij}^{\prime}\in \mathcal{Y}}w(x_{ij},l_{ij})l_\mathrm{bias}(l_{ij}^{\prime} \mid q,x_{ij})$ depends on $x_{ij}$.
Therefore, $\tilde{\mathcal{P}}$ denotes the feature distribution $\tilde{\mathcal{P}}(x_{ij})\propto \phi(x_{ij})P(x_{ij})$; thus,
\begin{align}
\begin{split}
    & \mathbb{E}_{x_{ij}\sim \mathcal{P}^2,l_{ij} \sim l_\mathrm{bias}(q, x_{ij})}\left[w(x_{ij},l_{ij})\cdot L(\hat{l}_q(x_{ij}), l_{ij})\right] \\
  &=C\cdot \mathbb{E}_{x_{ij}\sim \tilde{\mathcal{P}},l_{ij} \sim l_\mathrm{true}(q, x_{ij})}\left[L(\hat{l}_q(x_{ij}), l_{ij})\right],
\end{split}
\end{align}
where $C = E_{x_{ij}\sim \mathcal{P}^2}\left[\phi(x_{ij})\right]$, which completes the proof.
\end{proof}

\begin{algorithm}[t]
\caption{Training a fair ranking model for all pairwise constraints}
\label{al:proposed}
\begin{algorithmic}
\REQUIRE Learning rate $\eta$, number of loops $T$, training data $\mathcal{D}$, pairwise learning procedure $H$,
constraints $c^\mathrm{pair}_{11},\dots,c^\mathrm{pair}_{KK}$ corresponding to the pair of protected groups $(G_1,G_1),\dots,(G_K,G_K)$
\STATE Initialize $\lambda_{kl}=0$ for all $k,l$ and $w_{ij}=1$ for all $i,j$
\STATE Let $h = H(\mathcal{D},w_{ij})$
\FOR{$t=1,\dots,T$}
%\STATE calculate $\Delta_{kl} = \mathbb{E}_{q\sim Q}[\mathbb{E}_{x_{ij}\sim \mathcal{D}_\mathrm{pair}}[\langle \hat{l}_qx_{ij}), c^\mathrm{pair}(q, x_{ij})\rangle]]$ for all $k,l$.
\STATE Compute fairness violation $\Delta_{kl}$ using Eq.\eqref{eq:pairwise_bias}
\STATE Update $\lambda_{kl} = \lambda_{kl} - \eta \cdot \Delta_{kl}$ for all $k,l$
\STATE Let $\tilde{w}_{ij} := \exp{\left(\sum_{k,l\in [K]}\lambda_{kl} \mathbbm{1}\left[x_{ij}\in (G_k,G_l)\right] \right)}$ for all $i,j$
\STATE Compute weights $w_{ij}$ using Eq. \eqref{eq:weight}
\STATE Update model $h=H(\mathcal{D},w_{ij})$
\ENDFOR
\RETURN $h$
\end{algorithmic}
\end{algorithm}

Theorem~\ref{thm:unbiased} is the modest significant contribution of our study.
It states that the bias in observed pair labels can be simply and straightforwardly corrected by weighing the observed pairs.
Theorem~\ref{thm:unbiased} suggests that when we weigh the observed pairs, we trade-off the ability to train on unbiased pair labels to train on a slightly different distribution $\tilde{\mathcal{P}}$ over ${\mathcal X}^2$.
However, the change in the feature distribution does not affect the bias of the final learned model when given some mild conditions.
In these cases, training using the weighted pairs of items with biased pair labels is equivalent to training using the same pairs of items but with true pair labels.
Thus, in our experiments, we optimize the weighted KL-divergence loss on the paired dataset $\mathcal{D}_\mathrm{pair}$: 
\begin{align}
 \begin{split}
  \label{eq:exp_loss}
   \argmin_{h} \frac{1}{|\mathcal{D}_\mathrm{pair}|}\sum_{(q, x_{ij}, l_{ij}) \in \mathcal{D}_\mathrm{pair}} w(x_{ij},l_{ij})\left[D_{KL}(\hat{l}_q(x_{ij})\|l_{ij})\right].
 \end{split}
\end{align}
%Based on Theorem~\ref{thm:unbiased}, we can employ a Bayesian Personalized Ranking (BPR) loss \cite{DBLP:conf/uai/RendleFGS09} as a loss function $L$.
%The BPR loss only considers positive pair labels in its input, which is efficient for pairwise learning if the number of pairs $N^2$ is too large.
%The loss function is defined by:
%\begin{equation}
%\label{eq:bpr}
%  L(\hat{l}(x_{ij}), l_{ij}) = l_{ij}\log{\hat{l}(x_{ij})}.
%\end{equation}
%This form is an approximation of the AUC.
%A random subset $M$ of the pairs can be sampled if also there are too many positive instances in the training dataset.
%We use this form of the loss function in the experimental part of our study.
%Also, we tune the number of $M$ using the validation set.
\subsection{Determining Coefficients}
\label{subsec:coefficients}
We now describe how to learn the coefficients $\lambda$.
In practice, $K^2$ is often small in our approach, which is advantageous.
Thus, we propose to iteratively learn the coefficients so that the final model satisfies the desired pairwise constraints either on the training or validation set.
For simplicity, we first restrict the pairwise constraint as a pairwise statistical constraint.
Then, we explain how to apply our algorithm to other pairwise constraints.
%assume that $l_\mathrm{true}$ is known.

Intuitively, if the average exposure of $G_k$ is lower than the average exposure of $G_l$, the corresponding coefficients $\lambda_{kl}$ should be increased, %and  should be decreased 
(i.e., if we increase the weights of the pairs $(G_k, G_l)$ with positive pair labels and decrease the weights of the pairs $(G_k, G_l)$ with negative pair labels,
then the ranking model will be encouraged to rank items of $G_k$ higher than items of $G_l$, whereas items of $G_l$ are placed lower than items of $G_k$.
Both events will cause the difference in average exposure between $G_k$ and $G_l$ to reduce.
Thus, $\hat{l}$ will move closer to $l_\mathrm{true}$, namely, $h$ will move closer to the true, unbiased ranking model.

Accordingly, Algorithm \ref{al:proposed} works by iteratively performing the following steps.
(1) Evaluate the pairwise statistical constraint.
(2) Update the coefficients by subtracting the respective constraint violation multiplied by a step size.
(3) Compute the weights for each pair based on the multipliers using the closed-form provided by Proposition \ref{proposition}.
(4) Retrain the ranking model, given these weights.

Algorithm \ref{al:proposed} employs the pairwise ordering method procedure $H$, which takes pairs of items from dataset $\mathcal{D}$ and weights $w_{ij}$ to output a ranking model.
In practice, $H$ can be any pairwise training procedure that minimizes a weighted loss function over some parametric function class (e.g.,~\cite{DBLP:conf/sigir/CaoXLLHH06,DBLP:conf/uai/RendleFGS09,burges2010ranknet}).

Our resulting algorithm simultaneously minimizes the weighted loss and maximizes fairness by learning the coefficients.
These processes may be interpreted as competing goals with different objective functions.
Thus, it is a form of a nonzero-sum two-player game.
The use of nonzero-sum two-player games in fairness was proposed by \cite{DBLP:conf/alt/CotterJS19,DBLP:conf/aistats/JiangN20} for classification and  \cite{DBLP:conf/aaai/NarasimhanCGW20} for ranking and regression.
%We summarize our learning procedure in \ref{al:proposed}.

\subsection{Extension to Other Constraints}
\label{subsec:extension}
In the previous section, we assumed that $l_\mathrm{true}$ is known; now, we explain how to approximate it in practice.
%The initial restriction to pairwise constraints was made so that the values of the constraint functions $c^\mathrm{pair}_{kl}$ on any $x_{ij} \in \mathcal{X}^2 , l_{ij} \in \mathcal{Y}$ would be known.
%In general, the constraint functions depend on $l_\mathrm{true}$, which is unknown.
In general, the values of the pairwise constraint $c^\mathrm{pair}_{kl}$ depend on $l_\mathrm{true}$, which are unknown.
In these cases, we proxy the true pair label function $l_\mathrm{true}$ by observed pair label function $l_\mathrm{bias}$.
This allows us to evaluate the value of the constraint $c^\mathrm{pair}_{kl}$ using the observed pair label $l_{ij}$ and the group membership pair $(G_k, G_l), i\in G_k, j\in G_l$.
Then we propose applying the same technique described in Subsection~\ref{subsec:coefficients} that  iteratively weighs the loss to satisfy the desired constraint.
%In fact, all the pairwise constraints introduced in this study do not need any additional parameters.
Thus, we can apply other pairwise constraints in Algorithm \ref{al:proposed}.
\begin{figure*}[t]
\begin{tabular}{c}
 \begin{minipage}{0.25\hsize}
   \begin{center}
    \includegraphics[width=35mm]{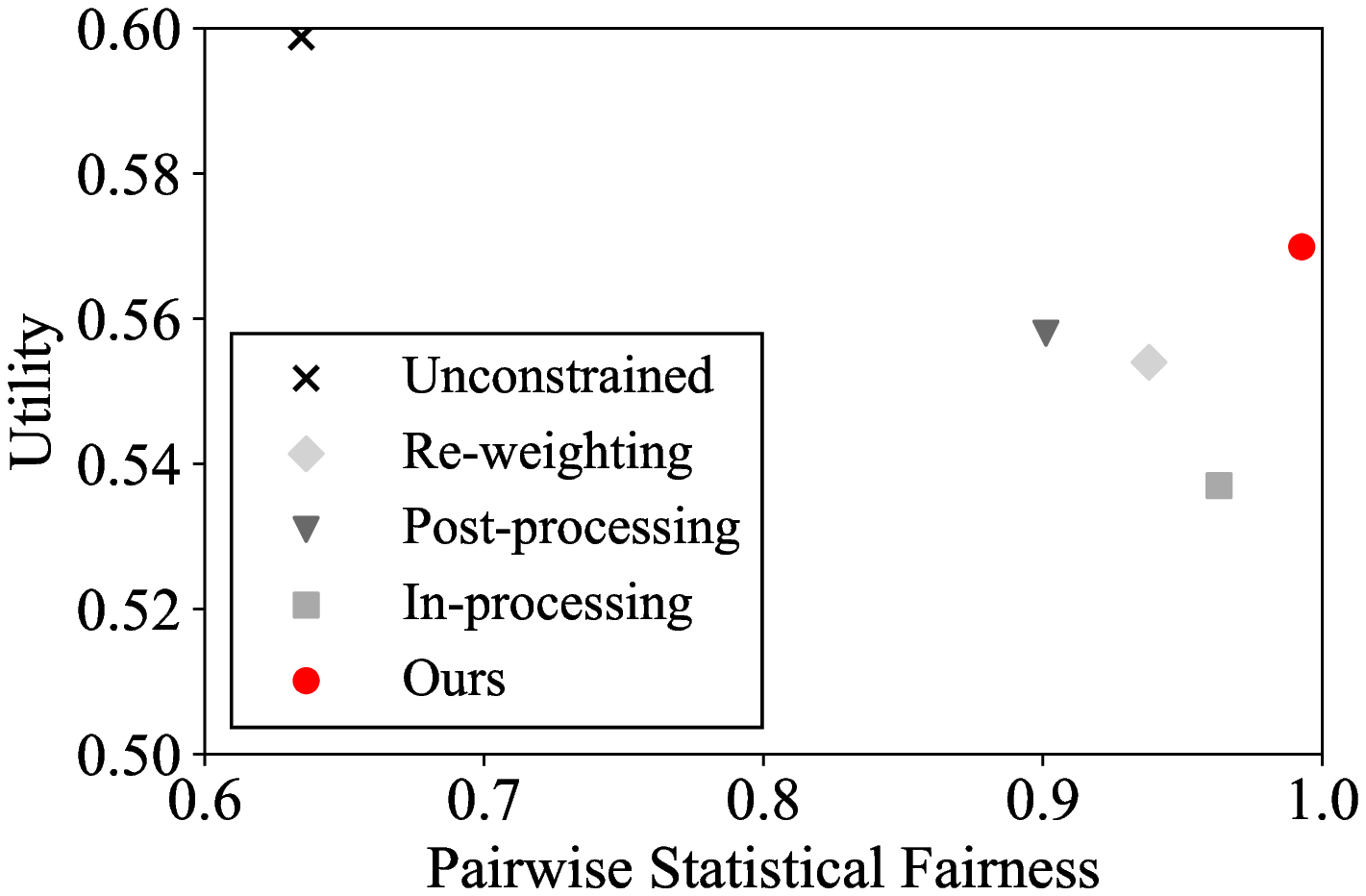}
   \end{center}
  \end{minipage}
  \begin{minipage}{0.25\hsize}
   \begin{center}
    \includegraphics[width=35mm]{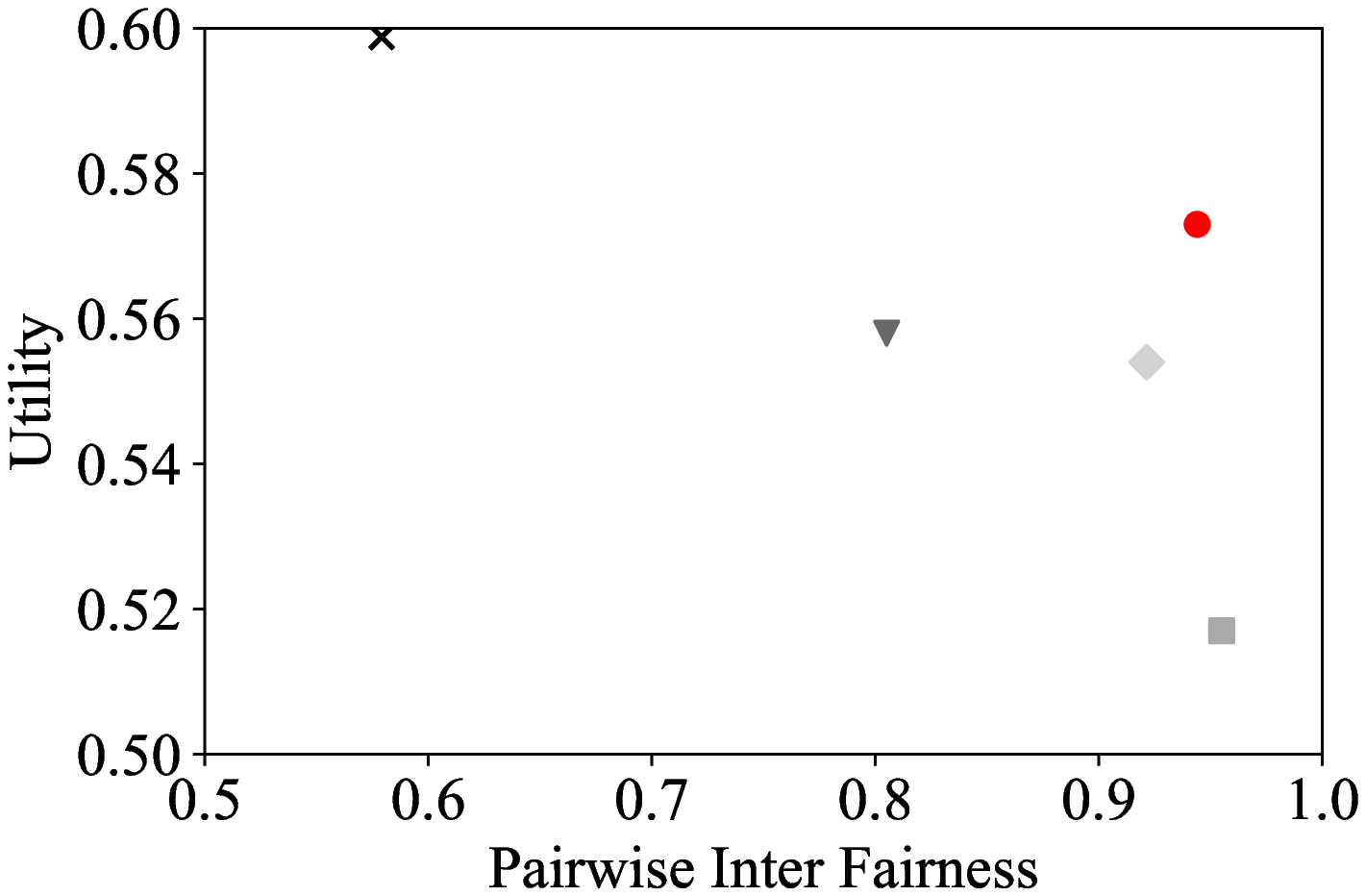}
   \end{center}
  \end{minipage}
  \begin{minipage}{0.25\hsize}
   \begin{center}
    \includegraphics[width=35mm]{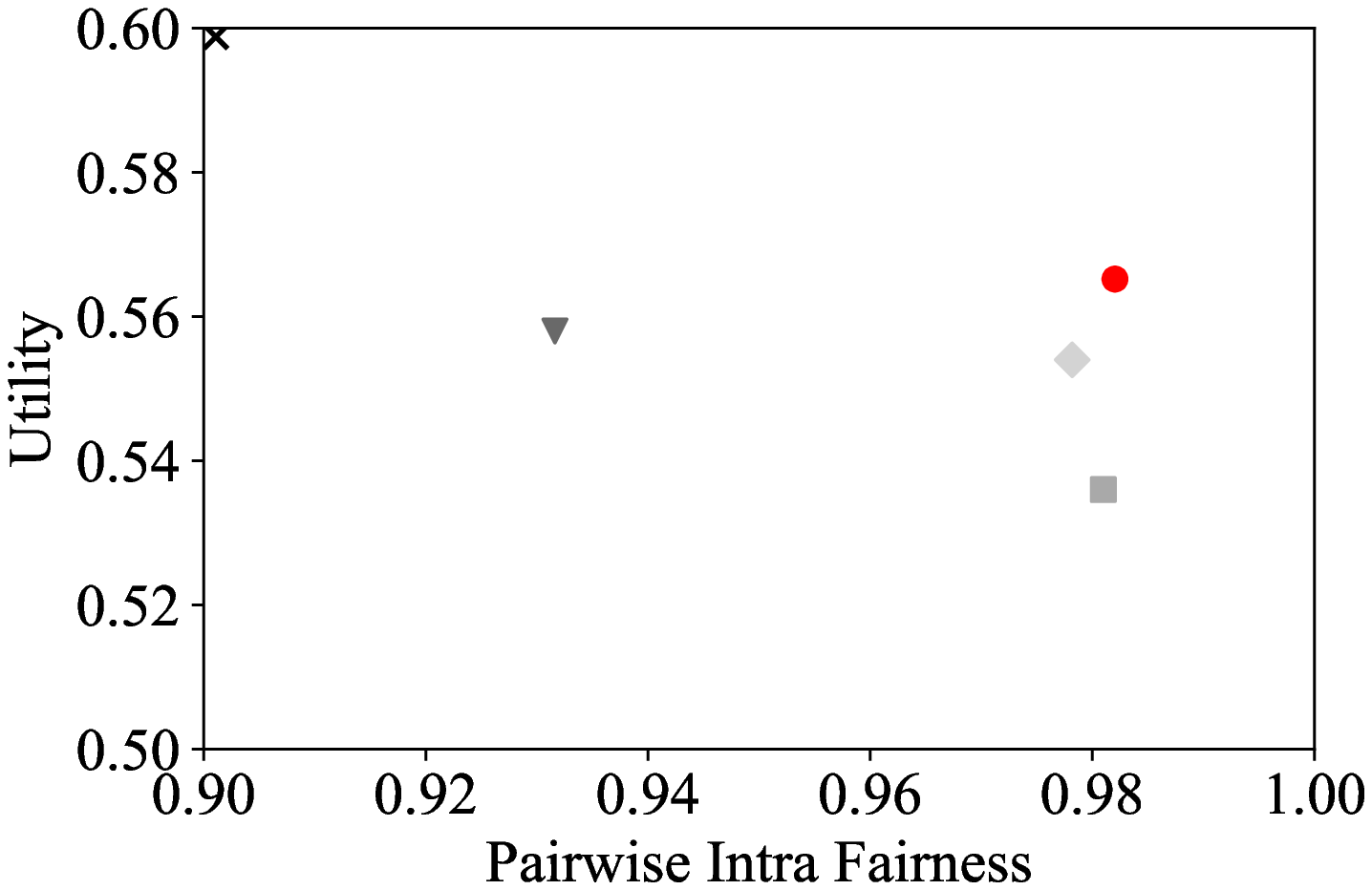}
   \end{center}
  \end{minipage}
  \begin{minipage}{0.25\hsize}
   \begin{center}
    \includegraphics[width=35mm]{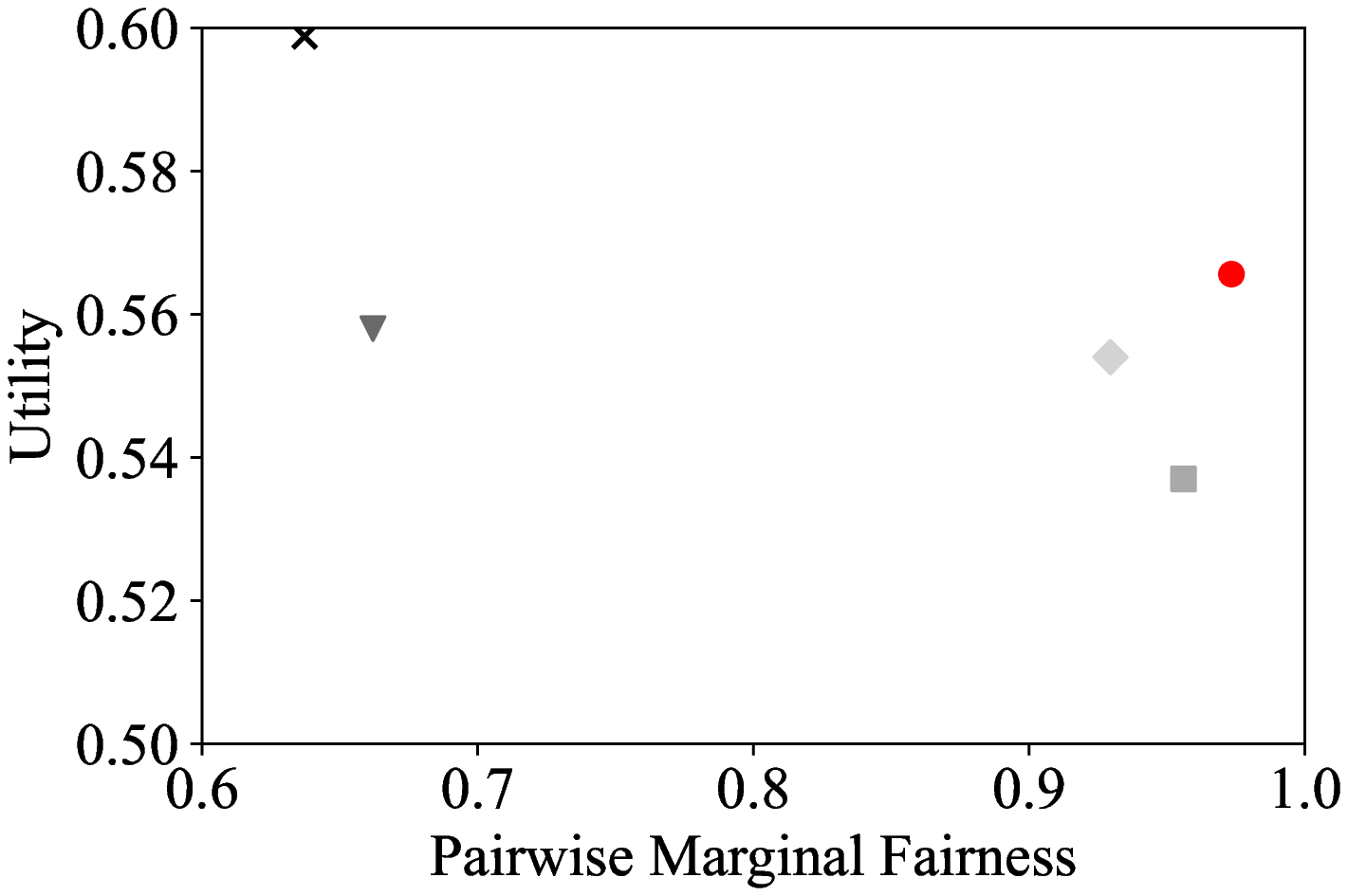}
   \end{center}
  \end{minipage}
\end{tabular}
  
\begin{tabular}{c}
  \begin{minipage}{0.25\hsize}
   \begin{center}
    \includegraphics[width=35mm]{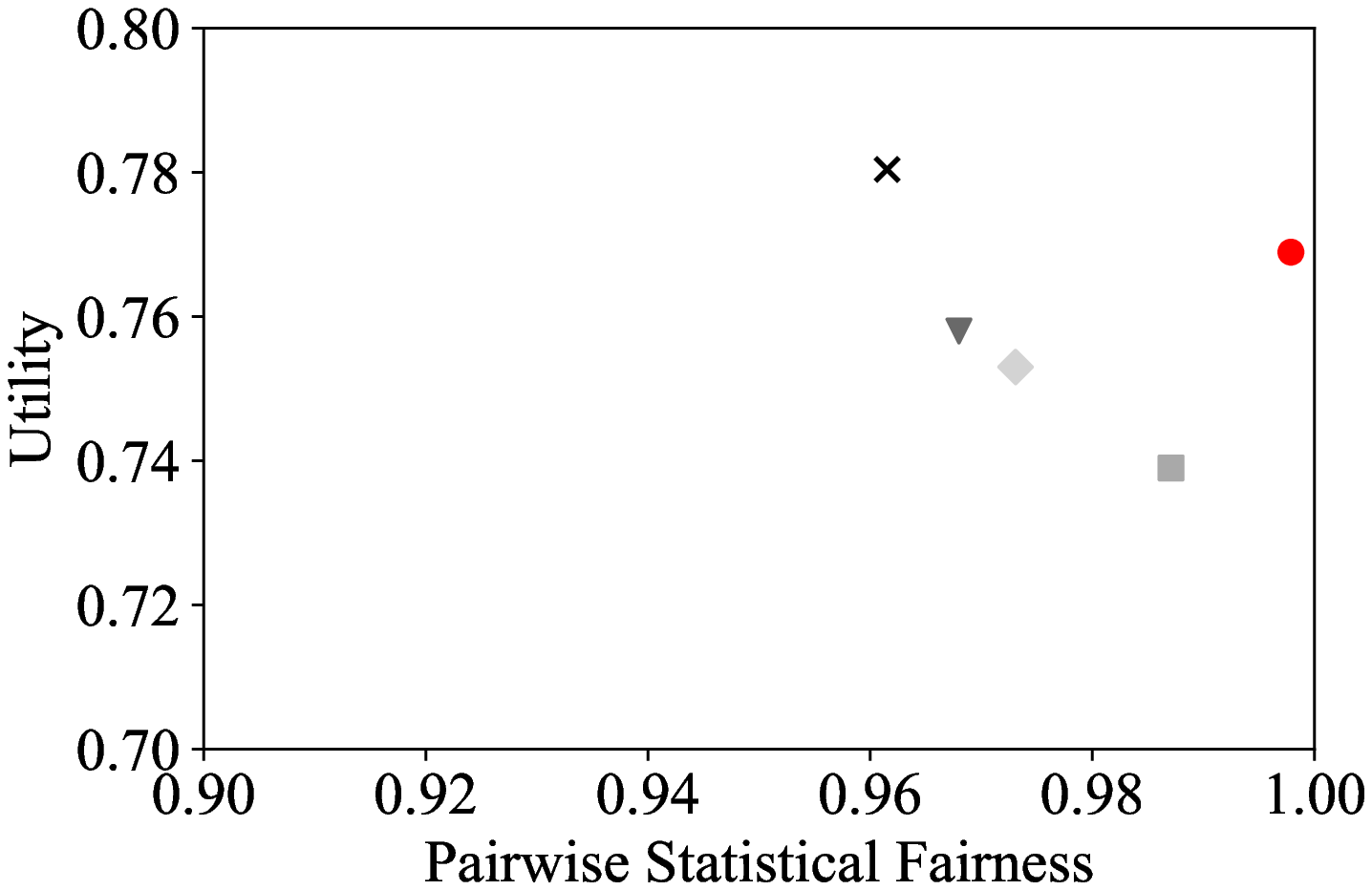}
   \end{center}
  \end{minipage}
  \begin{minipage}{0.25\hsize}
   \begin{center}
    \includegraphics[width=35mm]{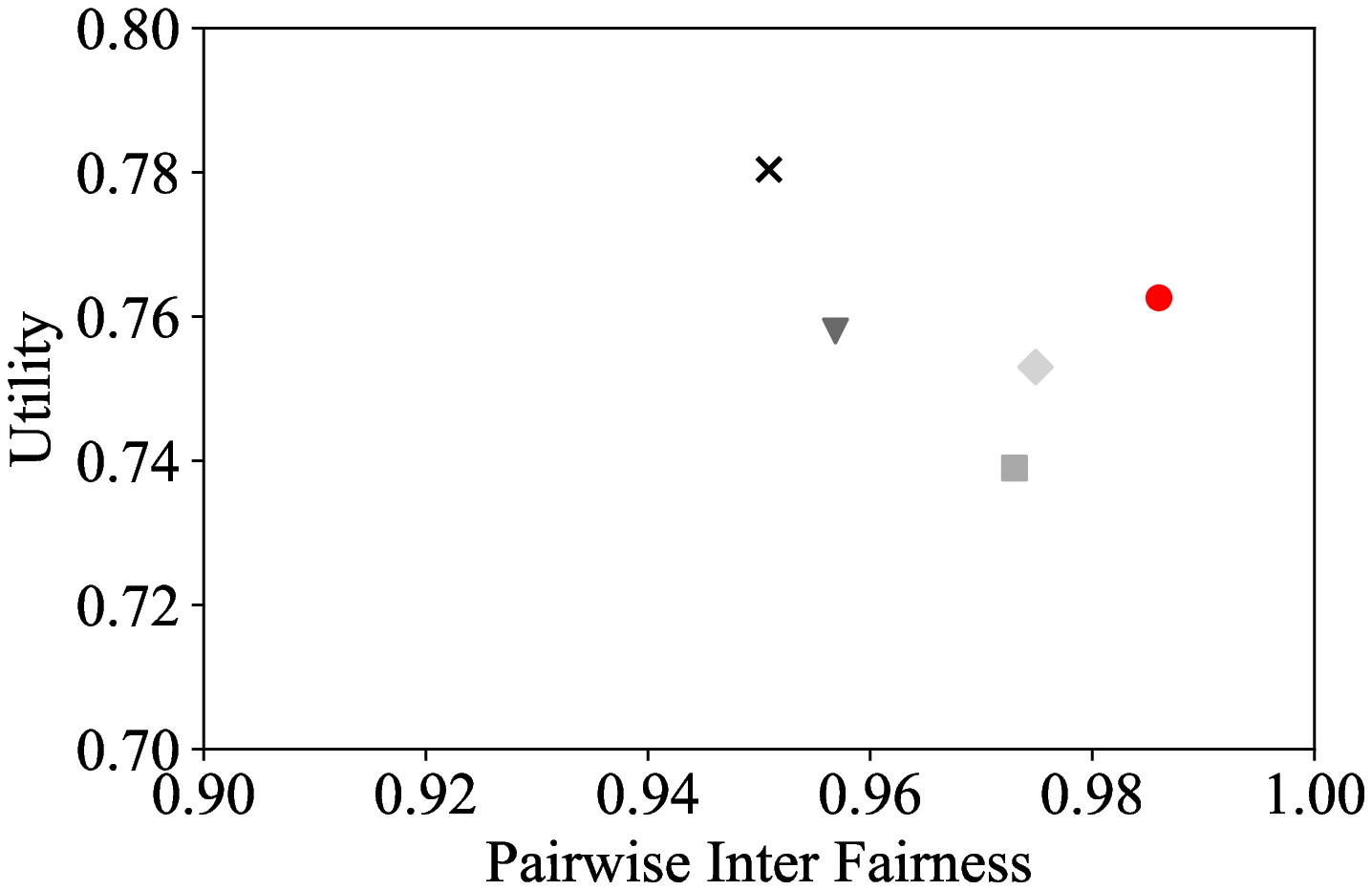}
   \end{center}
  \end{minipage}
  \begin{minipage}{0.25\hsize}
   \begin{center}
    \includegraphics[width=35mm]{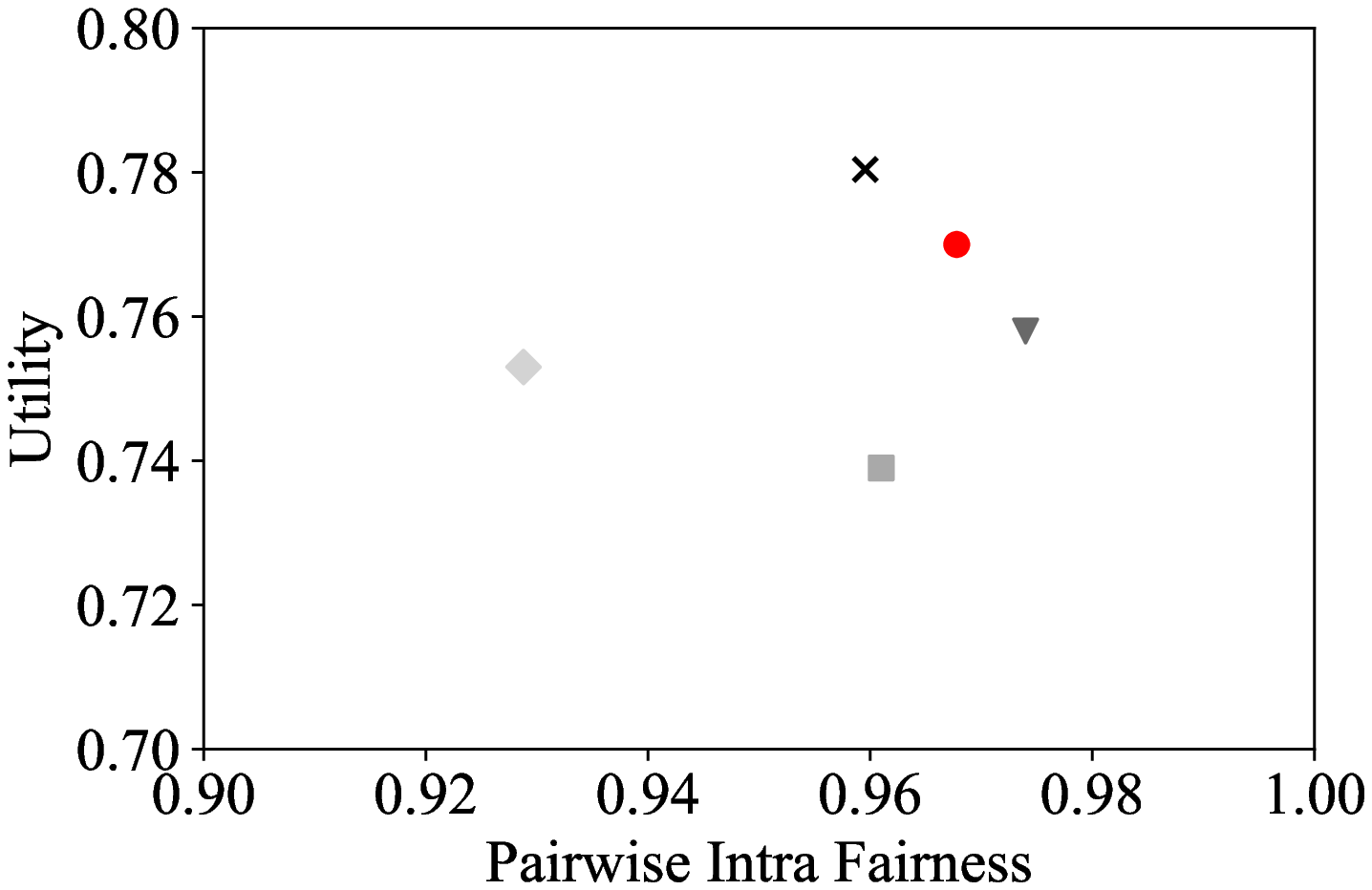}
   \end{center}
  \end{minipage}
  \begin{minipage}{0.25\hsize}
   \begin{center}
    \includegraphics[width=35mm]{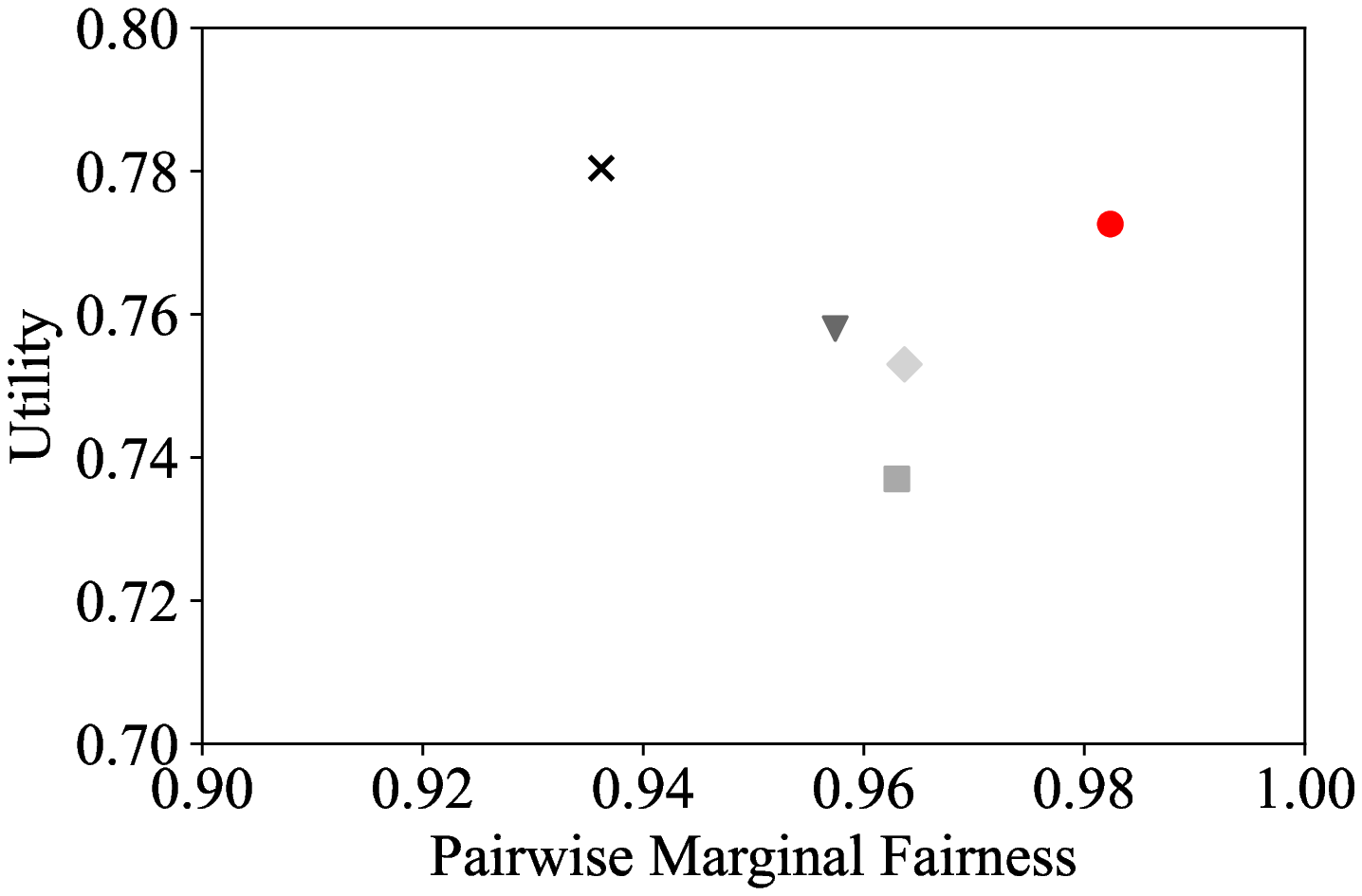}
   \end{center}
  \end{minipage}
\end{tabular}

\begin{tabular}{c}
  \begin{minipage}{0.25\hsize}
   \begin{center}
    \includegraphics[width=35mm]{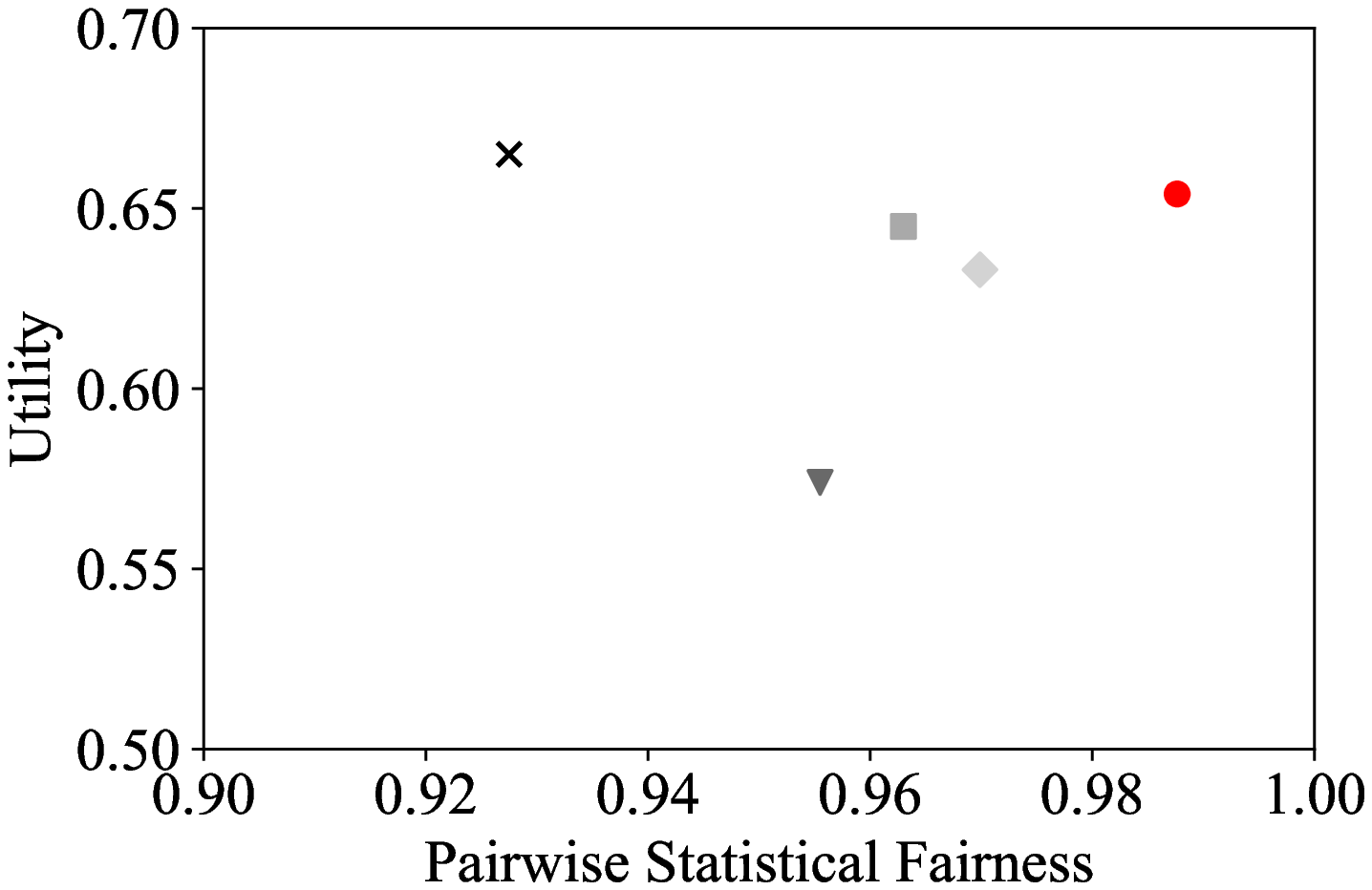}
   \end{center}
  \end{minipage}
  \begin{minipage}{0.25\hsize}
   \begin{center}
    \includegraphics[width=35mm]{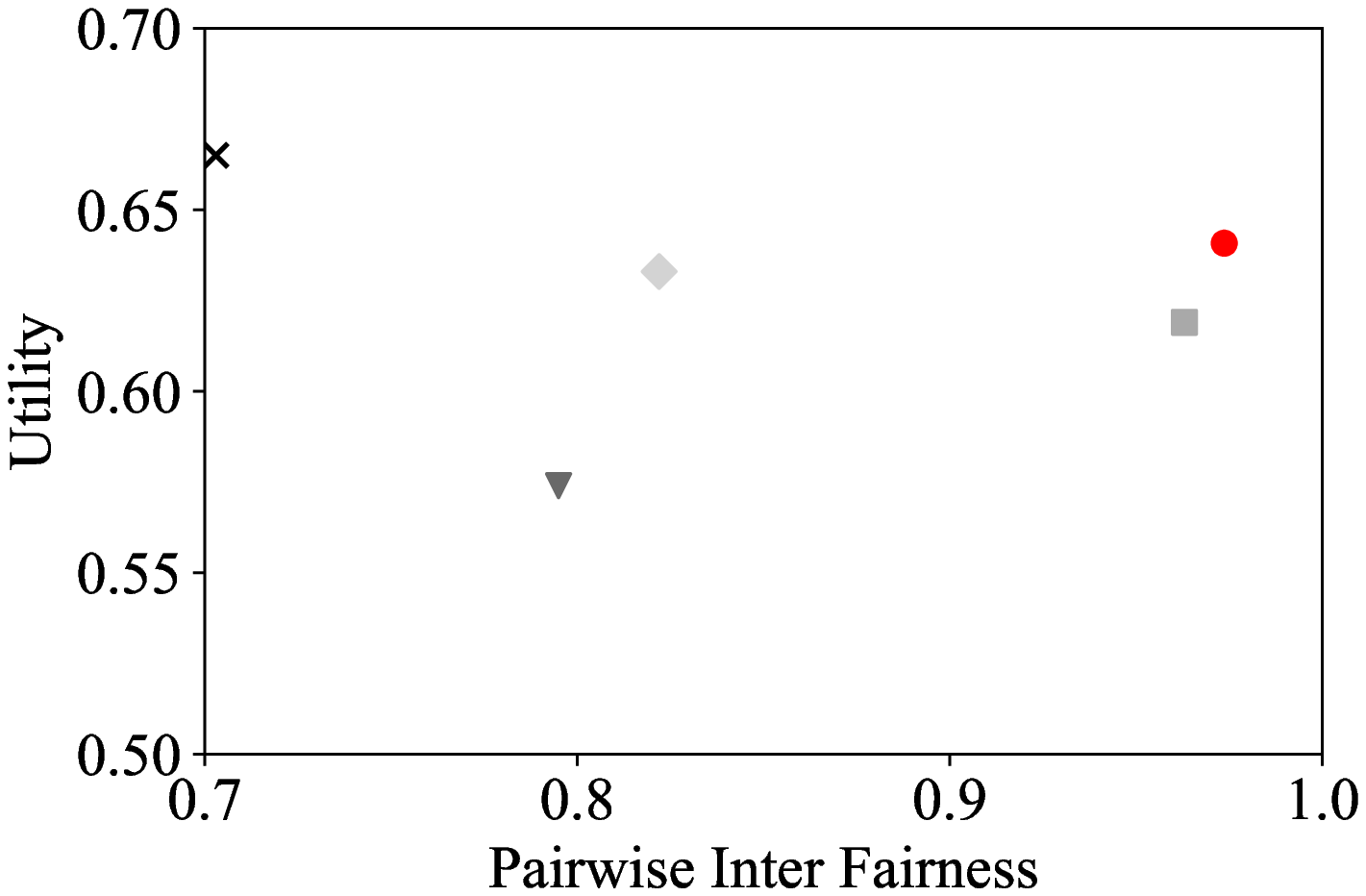}
   \end{center}
  \end{minipage}
  \begin{minipage}{0.25\hsize}
   \begin{center}
    \includegraphics[width=35mm]{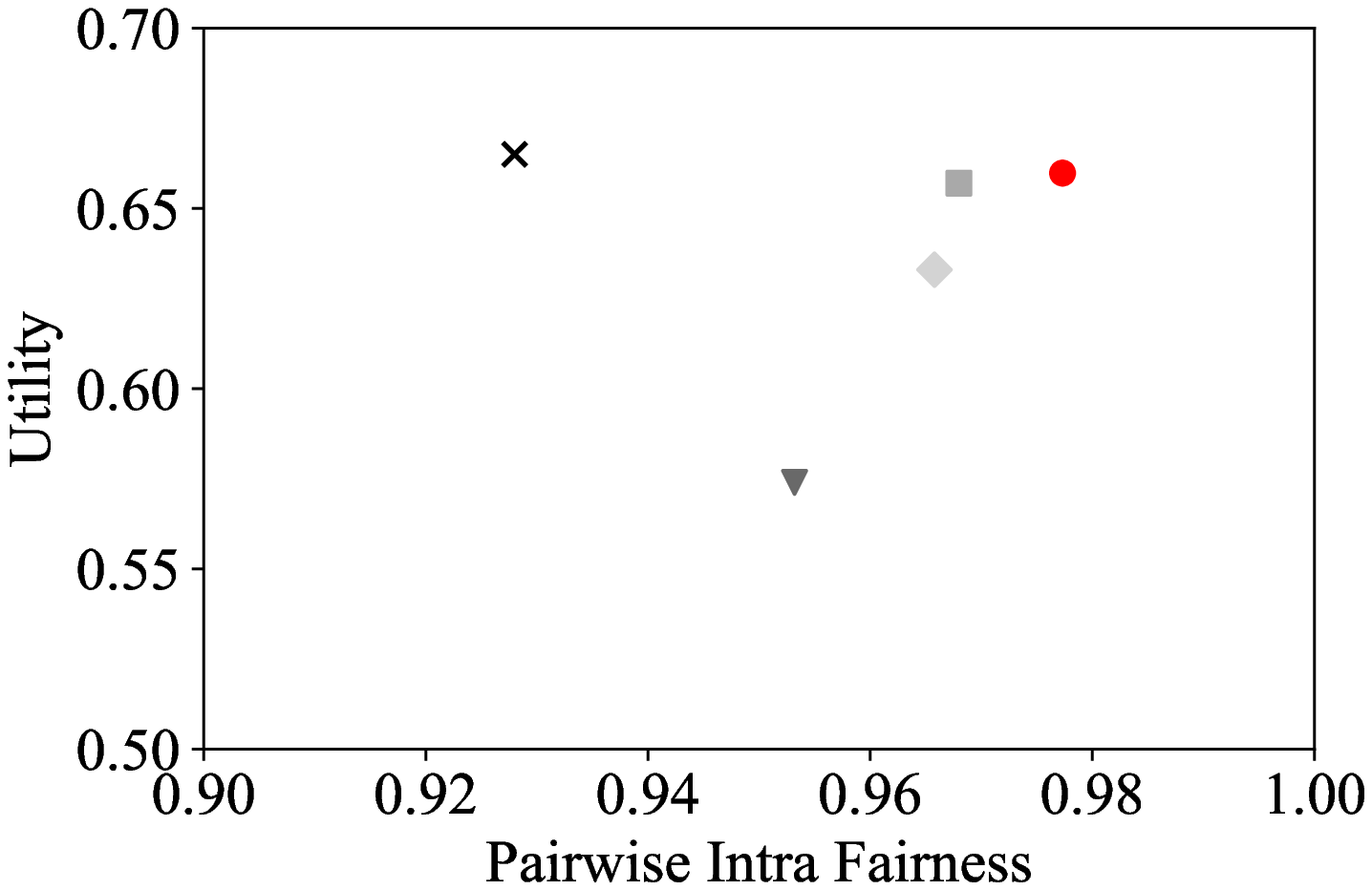}
   \end{center}
  \end{minipage}
  \begin{minipage}{0.25\hsize}
   \begin{center}
    \includegraphics[width=35mm]{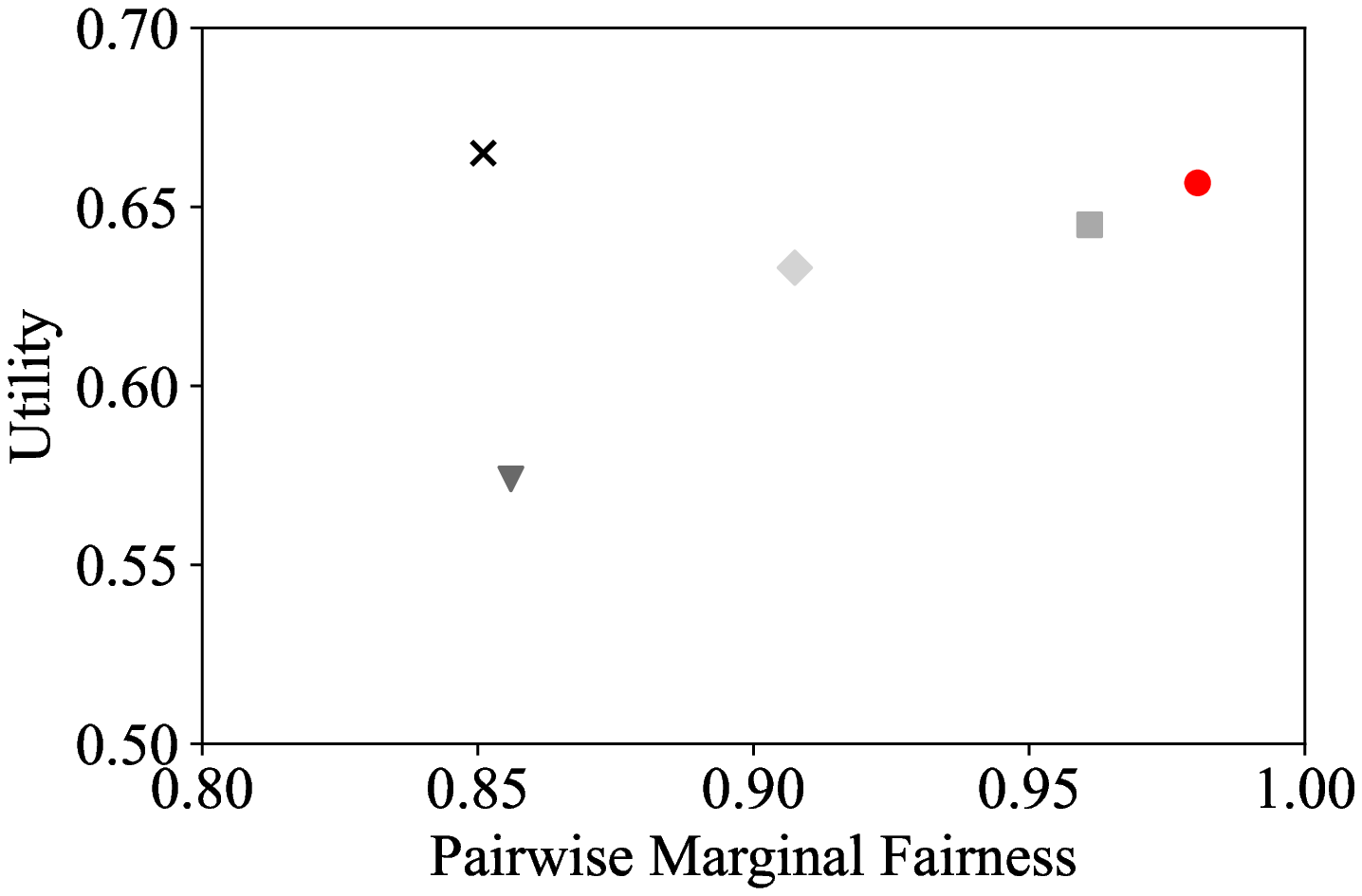}
   \end{center}
  \end{minipage}
\end{tabular}
  \caption{\textbf{Experiment results for fair ranking tasks}:
     Each column corresponds to a fairness measurement.
     Each row corresponds to a dataset:
     \textbf{Left} TREC.
     \textbf{Middle} ES.
     \textbf{Right} MSLR.
     In each graph, we show the best result for five methods: the unconstrained method, the re-weighting method \cite{DBLP:conf/aistats/JiangN20}, the post-processing method \cite{DBLP:conf/kdd/SinghJ18}, the in-processing method \cite{DBLP:conf/aaai/NarasimhanCGW20}, and our method.
     %For the re-weighting and our method, we use three different $\lambda$ by product of estimated $\lambda$ and $x\in \{0,1,2\}$.
     %For the post- and in-processing method, we chose three different $\lambda$ from $\{0,0.1,1\}$ and $\{0.005,0.05,0.5\}$, respectively.
     The best trade-off point is located in the upper right corner of each graph.
     All reported results are evaluated on the test set.
   }
  \label{fig:tradeoff}
\end{figure*}

\begin{figure*}[t]
 \begin{minipage}{0.33\hsize}
  \begin{center}
   \includegraphics[width=50mm]{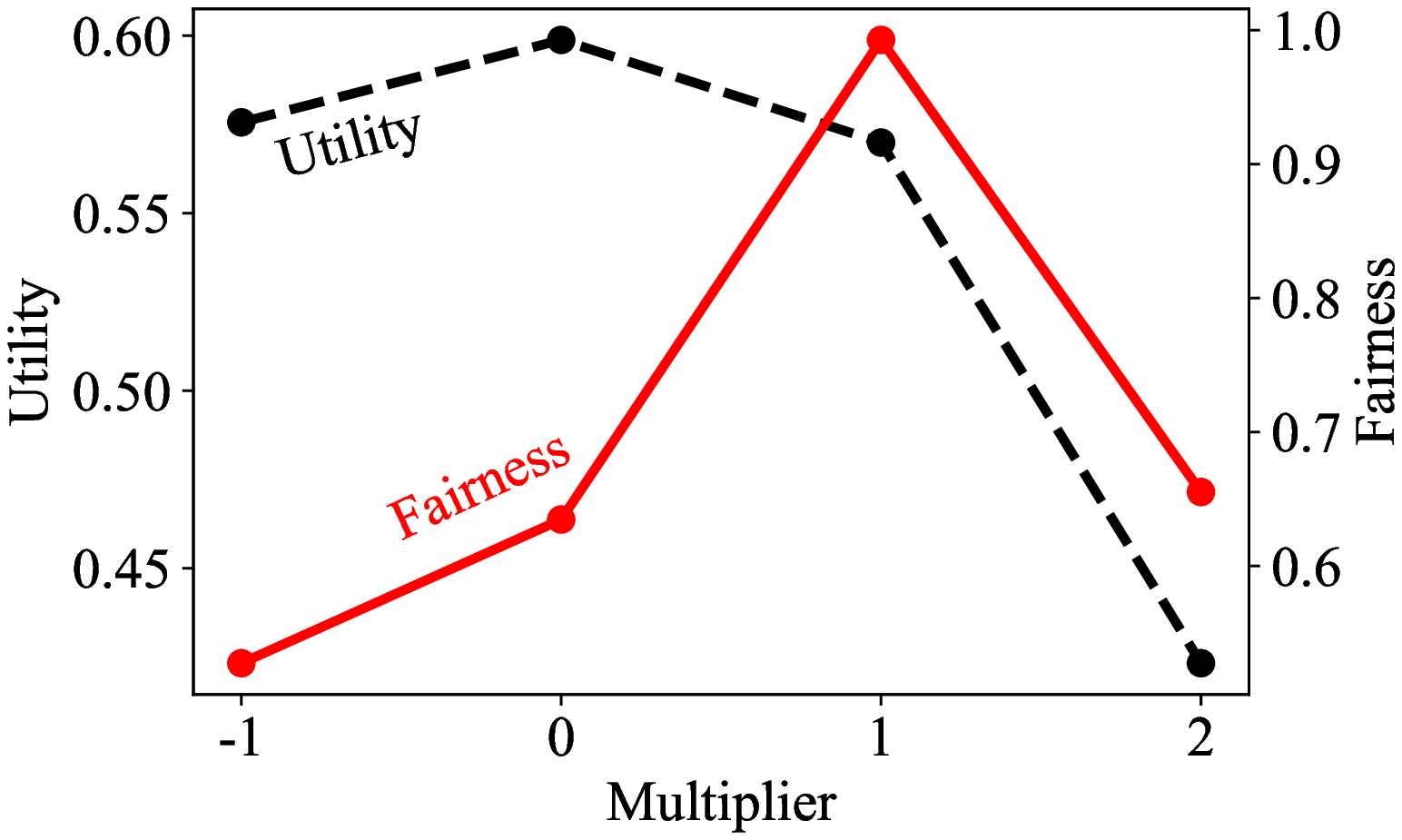}
  \end{center}
  \subcaption{TREC}
 \end{minipage}
 \begin{minipage}{0.33\hsize}
 \begin{center}
  \includegraphics[width=50mm]{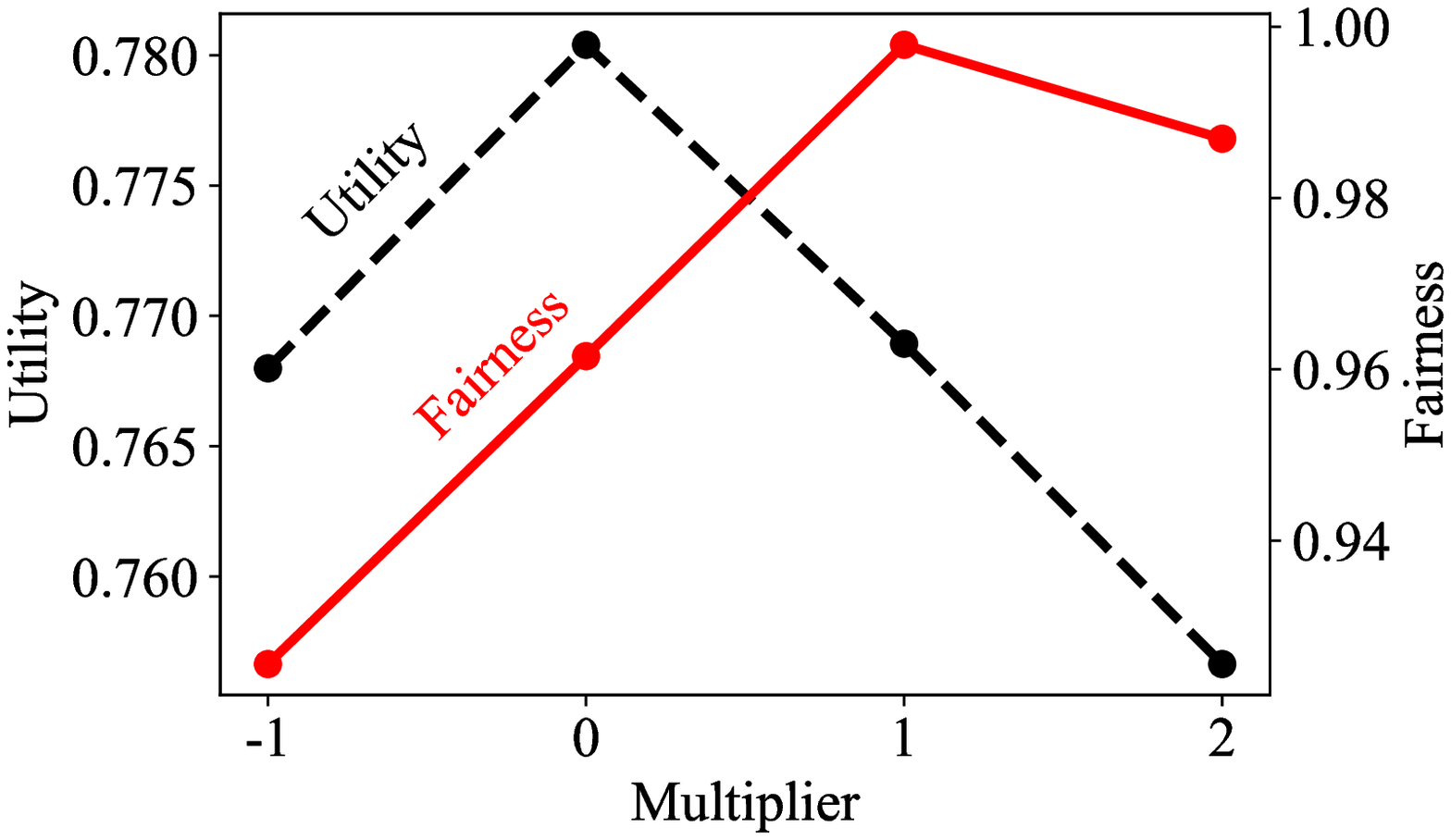}
 \end{center}
 \subcaption{ES}
 \end{minipage}
 \begin{minipage}{0.33\hsize}
 \begin{center}
  \includegraphics[width=50mm]{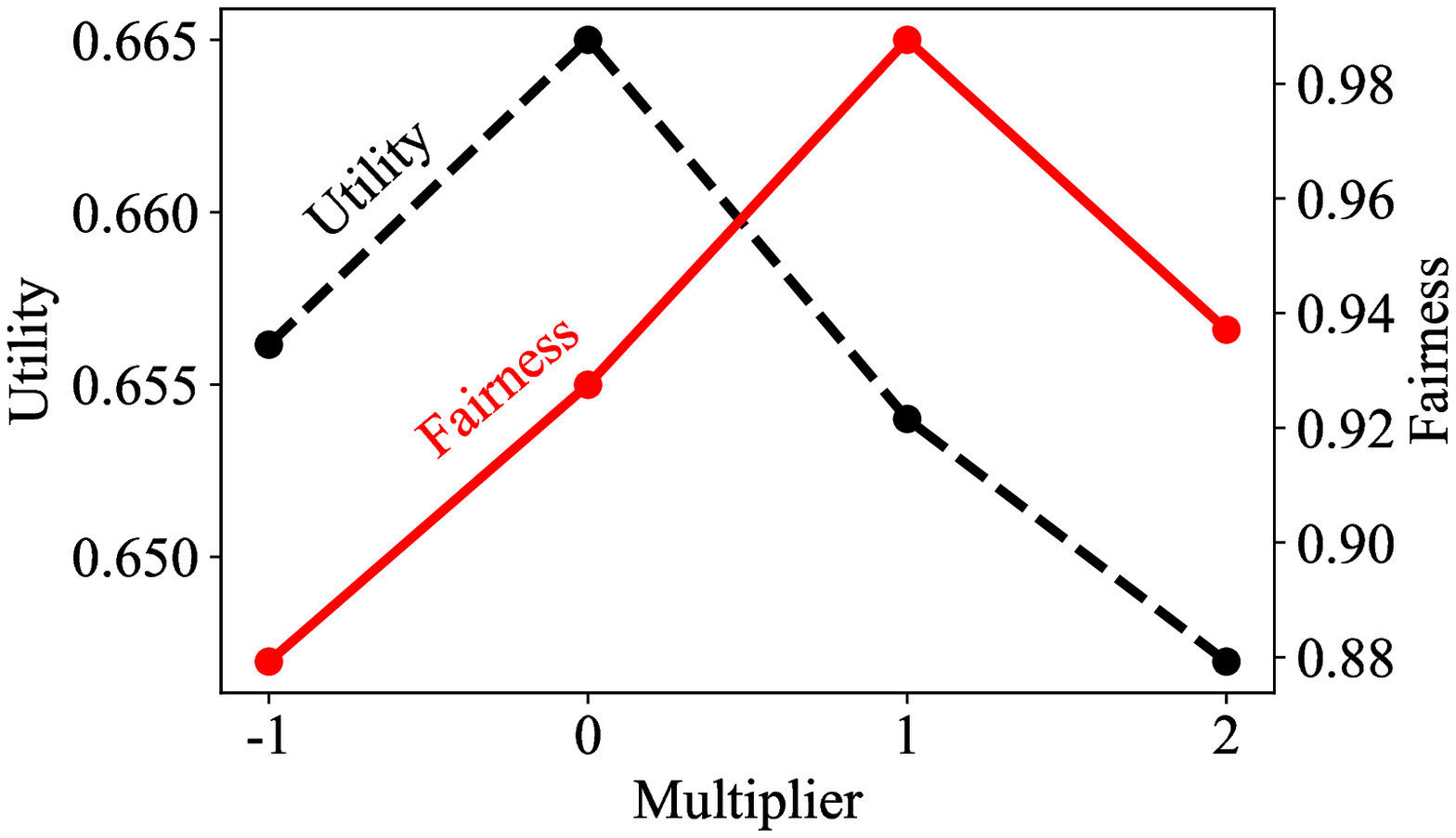}
 \end{center}
 \subcaption{MSLR}
   \end{minipage}
  \caption{Results of the trade-off curve for our method as the coefficient $\lambda$ changes.
   We show the test utility and fairness for pairwise statistical constraint as a function of changes in weights.
   }
  \label{fig:tradeoff_curve}
\end{figure*}

\section{Experiment}
\label{sec:experiment}
\subsection{Experimental Setup}
\label{subsec:experimental_setup}

\subsubsection*{Datasets}
\begin{table}[t]
 \caption{Summary of experimental dataset.}
 \label{tab:dataset}
  \begin{tabular}{cccc} \hline
   & TREC &  ES & MSLR \\ \hline \hline
   \# queries & 48 & 5 & 2805 \\ 
   \# items/query & 200 & 480 (ave.) & 132 (ave.) \\
   label& 
   \begin{tabular}{c}
   expertise\\ level
   \end{tabular}&
   \begin{tabular}{c}
   first-year\\ performance
   \end{tabular}&
   \begin{tabular}{c}
   relevance\\ score
   \end{tabular}\\
   protected. & gender & gender & quality score \\ \hline
  \end{tabular}
\end{table}
In our experiments, we use three real-world datasets for benchmark fair ranking tasks, summarized in Table~\ref{tab:dataset}.
These datasets are commonly used in fair ranking tasks and comprise query-item pairs.
%We also analyze the datasets in a single query task and a multi-query task.

The first dataset is the W3C experts (TREC Enterprise) dataset \cite{DBLP:conf/trec/CraswellVS05}, where the task is to rank candidates based on their expertise level on a topic using a corpus of their written e-mails.
This dataset has $48$ queries, and each query has $200$ items.
We preprocess the TREC dataset following~\cite{DBLP:conf/cikm/ZehlikeB0HMB17}.
Each item is described by a feature vector comprising five attributes with numerical and categorical features and a label ($y=1$ if the expertise level is higher than the median; otherwise, $y=0$).
We consider two groups based on their gender (male or female).

The second dataset is the Engineering Students (ES) dataset, where the task is to rank students based on their average first-year performance using admission test scores.
This dataset has $5$ queries, and each year has an average of $480$ items.
We use the preprocessed ES dataset provided by an existing fair ranking study~\cite{DBLP:conf/www/Zehlike020}, as the original dataset is not publicly available.
Each item is described by a feature vector comprising five attributes with numerical and categorical features and a label ($y=1$ if the average first-year performance is above the median; otherwise, $y=0$).
We consider two groups based on their gender (male or female).

The third dataset is the Microsoft Learning to Rank (MSLR) dataset~\cite{qin2013introducing}, where the task is to rank web pages based on their relevance scores to a query using features of the web pages.
This dataset has $2805$ queries, and each query has an average of $132$ students.
We preprocess the MSLR dataset following \cite{DBLP:conf/sigir/YadavDJ21}.
Each item is described by a feature vector comprising $135$ attributes with numerical and categorical features and a label (if the relevance score is $2$ or more, $y=1$, otherwise $y=0$).
We consider five groups based on their quality scores, with the same number of items for each group.

We randomly split the queries into training and test sets with a ratio of $4:1$.
We split the training and validation sets in a ratio of $4:1$, where the validation set is used to tune hyperparameters.
We evaluate all metrics for individual queries and report the average across queries in the test set.
%The reported results are averaged over all folds of the test set.
\subsubsection*{Baselines}
%We detail the baselines for our experiments.
We compare our method against the following four methods:
(1) the unconstrained method that minimizes unweighted pairwise loss function to optimize only utility and not fairness,
(2) the re-weighting method \footnote{https://github.com/google-research/google-research/tree/master/label\_bias} that minimizes pointwise loss function Eq.~\eqref{eq:point_loss} by weighting each label to satisfy the equal opportunity constraint~\cite{DBLP:conf/aistats/JiangN20},
%This method estimates co-efficients parameters $\lambda$ for each group, and uses the co-efficients for its weights.
(3) the post-processing method \footnote{https://github.com/ashudeep/Fair-PGRank} on estimated labels from a fairness-oblivious linear regression to satisfy the disparate impact constraint \cite{DBLP:conf/kdd/SinghJ18}, which solves a linear problem (LP) per query with some slack $\lambda$ in the constraint, and 
(4) the in-processing method \footnote{https://github.com/google-research/google-research/tree/master/pairwise\_fairness} that employs the Lagrangian approach and jointly learns the model parameters and Lagrange multipliers to satisfy the pairwise fairness \cite{DBLP:conf/aaai/NarasimhanCGW20}, which uses a hinge relaxation to make the constraints differentiable and adjusts a slack $\lambda$ for the constraints.
We employ a linear ranking model $h$ and use Adam \cite{DBLP:journals/corr/KingmaB14} for gradient updates in all methods in the experiments.
%Also, we use three different $\lambda$ for all of the methods to create the trade-off curve.

\subsubsection*{Hyperparameters}
We tune the hyperparameters across all experiments using the validation set.
Using the validation set, we tune the learning rate of Adam within the range of $[10^{-1},10^{-4}]$ for all methods. %except the post-processing method.
For the unconstrained methods, we tune the number of training epochs within the range of $[10,50]$ and the minibatch size within the range of $[2^8,2^{11}]$.
For the re-weighting method, we fix the learning rate for the coefficients to $\eta = 1$ and tune the number of loops $T$ within the range of $[50,100]$.%tune the number of training epochs within the range of $[100,500]$.
%, and we use three different $\lambda$ by product of estimated $\lambda$ and $x\in \{0,1,2\}$ for the re-weighting method. That is, $x=0$ corresponds to an unweighted method, $x=1$ to a appropriately weighted method, and $x=2$ to an overweighted method.
For the post-processing method, we chose the smallest slack for its constraint in increments of $10\%$ until the LPs are feasible.%we chose $\lambda \in \{0,0.1,1\}$ as a small, appropriate, and large slack for the disparate impact constraint.
For the in-processing method, we fixed the number of training epochs at $2,500$ and chose the smallest slack for its constraint in increments of 5$\%$ until the method returned a nondegenerate solution.
In addition, we tune the minibatch size within the range of $[2^4,2^7]$ for the in-processing method.
%we chose  $\lambda \in \{0.005,0.05,0.5\}$ as a small, appropriate, and large slack for the pairwise constraint.
For our method, we fix the learning rate for the coefficients to $\eta = 1$ and tune the number of loops $T$ within the range of $[10,50]$.
%We use three different $\lambda$ by product of estimated $\lambda$ found by Algorithm \ref{al:proposed} and $x\in \{0,1,2\}$ for our method. That is, $x=0$ corresponds to an unweighted method, $x=1$ to an appropriately weighted method, and $x=2$ to an overweighted method.
\subsubsection*{Evaluation Metrics}
\label{subsec:evaluation}
We use AUC (Eq.~\eqref{eq:auc}) as the utility measurement in all experiments because we want to examine the trade-off between utility and fairness in a pairwise manner.
For all experiments, we evaluate our procedures with respect to the pairwise statistical, inter, intra, and marginal constraints.
These constraints are computed according to~\cref{def:pair-statistical,def:pair-inter,def:pair-intra,def:pair-marginal}, but we do not have access to the true label function.
Instead, we approximate the true pair label function $l_\mathrm{true}$ using the observed pair label function $l_\mathrm{bias}$ in the constraints.
%Finally, we report the worst fairness (the difference between the absolute value of the largest disparity $c^\mathrm{pair}_{kl}$ and its theoretical upper bound) for all group pairs.
%Finally, we report the pair constraints are calculated for each pair, and the difference between the pair and the ideal value is reported.
%In other words, we evaluate the fairness value with a minimum of 0 and a maximum of 1 in nuclear experiments.
In addition, we report the fairness of each experiment in the range $[0,1]$ using the following equation that returns $1$ if the model is completely fair.
\begin{equation}
  \label{eq:worst_fairness}
  1 - \underbrace{\max \left\{ \Delta_{kl} - \Delta_{lk} \mid k,l\in [K] \right\}}_\text{fairness\  violation}.
\end{equation}

 %Demo-List \eqref{eq:demo-list}, Tre-List \eqref{eq:tre-list}, Imp-List \eqref{eq:imp-list}, Mar-Pair \eqref{eq:mar_pair}, Inter-Pair \eqref{eq:inter_pair}, and Intra-Pair \eqref{eq:intra_pair} fairness.
%The post-processing and the in-processing method cannot be readily applied to pairwise-style constraints \eqref{eq:pair_fair} and listwise-style constraints \eqref{eq:list_fair}.
%Therefore, we report the results from these methods that provide the best fairness if fairness constraints cannot be applied to these methods.

\subsection{Results}
We present the results in Figure \ref{fig:tradeoff}.
Our method often yields a ranking model with the highest test fairness among the five methods (Figure \ref{fig:tradeoff}).
The results suggest that the fairness in ranking is significantly improved by weighting the observed data in a pairwise manner.
We also include test utility in the results.
Our method can effectively optimize the trade-off between utility and fairness although the major goal of fair methods is to produce fair outputs.

In addition, the results highlight the disadvantages of existing methods for producing a fair ranking.
The re-weighing method often yields a poor trade-off between fairness and utility.
As the re-weighing method uses pointwise ordering method that ignores the query-level dataset structure, its loss function is dominated by queries with a large number of items~\cite{DBLP:journals/ftir/Liu09}. 
Moreover, the weighting procedure of this method does not consider the order of items, resulting in suboptimal fairness in ranking~\cite{DBLP:conf/kdd/BeutelCDQWWHZHC19}.
The in-processing method employs pairwise ordering method; however, its results do not consistently provide a fair ranking.
%Interestingly, the pairwise method worse than pointwise method.
The in-processing method suffers from overfitting as it requires approximating nondifferentiable constraints with a certain amount of slack to make gradient-based training possible~\cite{DBLP:conf/icml/CotterGJSSWWY19}.
The post-processing method often failed to improve fairness well, even with tuning its slack for the constraints.
%This is because the post-processing method is misled by the estimated labels, which are the outputs from the model trained on the biased labels with no constraint.
Further, the post-processing method performs with poor utility when it requires a large slack in the constraints to solve the LP per query feasibly.
%This is because this dataset has intersectional groups, and extremely small groups are generated.

\subsection{Results for Changes in Coefficients}
\label{subsec:insight}
We now validate our algorithm for learning the coefficients.
Our method does not require a hyperparameter to adjust the fairness constraints.
Instead, we train a model using appropriate weights of the training items generated by Algorithm~\ref{al:proposed}.
Thus, we investigate whether the weights obtained by this algorithm can improve fairness.
To investigate, we consider the optimal coefficients $\lambda = \lambda^*$ found by Algorithm~\ref{al:proposed}.
Then we train a model on weighted training pairs by changing $\lambda$.

In Figure \ref{fig:tradeoff_curve}, we illustrate the test utility and fairness for the pairwise statistical constraint as a function of changes in weights.
We only show the results of the pairwise statistical constraint as we have similar observations on other constraints.
For each constant value, $x$, on the $x$-axis, we train a ranking model with the data weights based on the setting $\lambda = x\cdot \lambda^*$.
Then, we plot the utility and fairness.

For $x=1$, corresponding to our method, the best fairness is achieved (Figure \ref{fig:tradeoff_curve}).
The results suggest that our algorithm can correct the biases appropriately.
Meanwhile, $x=0$, which corresponds to training on the unweighted pairs, gives us the highest utility.
Remarkably, for $x=2$, the results show the lowest utility and poor fairness.
We believe this is because too many irrelevant items are placed at the top simply as they belong to a specific group. %a socially salient group.
%For $x=-1$, the method has the worst fairness.
%This is because 
%These results give further evidence of the ability of our method to effectively train on the underlying, true pair labels despite only observing biased pair labels.
\section{Conclusions}
\label{sec:conclude}
In this paper, we presented a pre-processing method based on pairwise ordering method for fair ranking.
Our pairwise ordering method can identify biases in a pair of groups by assuming that there exists a pair of ground-truth unbiased data.%by detecting bias in pairs of items.
Our method for correcting these biases is based on weighting the training pairs of items.
We showed that a ranking model trained on the weighted data is unbiased to the ground truth.
Experimentally, we showed that our method yields a fair ranking model in various fairness contexts.
These results demonstrate the advantage of our pre-processing method in producing a fair ranking.
%In future work, we will consider biases of data in a listwise manner rather than in a pairwise manner.
In the future, we will consider extending our method to address individual fairness.

\begin{acks}
We would like to thank Hiroya Inakoshi for his in-depth feedback on this paper.
We would also like to thank Enago (www.enago.jp) for the English language review.
\end{acks}

\bibliographystyle{ACM-Reference-Format}
\bibliography{ref.bib} %{C:/texlive/texmf-local/bibtex/bib/local/ref}

\end{document}